\documentclass{article} % For LaTeX2e
\usepackage{iclr2025_conference,times}

\iclrfinalcopy

\usepackage{amsmath,amsfonts,bm}

\def\Figref#1{Figure~\ref{#1}}

\def\Secref#1{Section~\ref{#1}}
\def\eqref#1{equation~\ref{#1}}
\def\1{\bm{1}}

\def\gemma{{Gemma}\xspace}
\def\llama{{Llama3}\xspace}
\def\mistral{{Mistral}\xspace}

\def\rvx{{\mathbf{x}}}
\def\rvy{{\mathbf{y}}}

\DeclareMathAlphabet{\mathsfit}{\encodingdefault}{\sfdefault}{m}{sl}
\SetMathAlphabet{\mathsfit}{bold}{\encodingdefault}{\sfdefault}{bx}{n}

\DeclareMathOperator*{\argmax}{arg\,max}

\usepackage{tabulary}
\usepackage{tabularx}
\newcolumntype{Y}{>{\raggedright\arraybackslash}1.3X}
\newcolumntype{Z}{>{\raggedright\arraybackslash}1.5X}

\usepackage{url}
\usepackage{microtype}
\usepackage{graphicx}

\usepackage{booktabs}
\usepackage{algorithm}
\usepackage{algorithmic}
\usepackage{comment}
\usepackage{wrapfig}
\usepackage{subcaption}

\usepackage{paralist}

\usepackage[breaklinks=true]{hyperref}

\usepackage[utf8]{inputenc}   % Handles UTF-8 encoding
\usepackage[T1]{fontenc}      % Uses T1 font encoding
\usepackage{listings}
\usepackage{xcolor}
\definecolor{ForestGreen}{RGB}{34,139,34}

\definecolor{darkblue}{rgb}{0.0, 0.0, 0.5}

\hypersetup{
    colorlinks=true,
    linkcolor=darkblue,      % Color for internal links (sections, pages)
    citecolor=darkblue,      % Color for citation links (bibliography)
    urlcolor=darkblue        % Color for external links (URLs)
}

\usepackage{xspace}

\newcommand{\eg}{e.g.\@\xspace}

\newcommand{\ie}{i.e.\@\xspace}

\usepackage{amsmath}
\usepackage{amssymb}
\usepackage{mathtools}
\usepackage{amsthm}

\usepackage{booktabs}    % For professional-looking tables
\usepackage{multirow}    % For multi-row cells if needed
\usepackage{array}

\usepackage[capitalize,noabbrev]{cleveref}

\theoremstyle{plain}
\newtheorem{proposition}{Proposition}%[theorem]

\newtheorem{definition}{Definition}
\newtheorem{assumption}{Assumption}
\theoremstyle{remark}

\newcommand{\baselineM}{{M}}
\newcommand{\changedM}{{M'}}

\newcommand{\baselineDist}{{P_B^{\baselineM}}}
\newcommand{\changedDist}{{P_B^{\changedM}}}

\title{An Auditing Test to Detect Behavioral Shift in Language Models\thanks{This work was published at ICLR 2025.}}

\author{Leo Richter\textsuperscript{\rm 1} \qquad Xuanli He\textsuperscript{\rm 1} \qquad Pasquale Minervini\textsuperscript{\rm 2,3} \qquad Matt J. Kusner\textsuperscript{\rm 4,5} \\
\textsuperscript{\rm 1}UCL Centre for Artificial Intelligence, University College London, United Kingdom \\
\textsuperscript{\rm 2}School of Informatics, University of Edinburgh, United Kingdom \\
\textsuperscript{\rm 3}Miniml.AI, United Kingdom \\
\textsuperscript{\rm 4}Polytechnique Montréal, Canada \\
\textsuperscript{\rm 5}Mila – Quebec AI Institute, Canada \\
\texttt{ucablri@ucl.ac.uk, matt.kusner@mila.quebec} \\
}

\begin{document}

\maketitle

\begin{abstract}
As language models (LMs) approach human-level performance, a comprehensive understanding of their behavior becomes crucial to avoid potential harms. 
%This includes evaluating capabilities, biases, task performance, and alignment with societal values.
%
While extensive initial evaluations, including red teaming and diverse benchmarking, can establish a behavioral profile, 
subsequent fine-tuning or deployment modifications may alter these model behaviors in unintended ways.
We study the \emph{behavioral shift auditing} problem, where the goal is to detect unintended changes in model behavior. 
We formalize this problem as a sequential hypothesis test.  
We apply and extend a recent testing method to include a configurable tolerance parameter that adjusts sensitivity to behavioral changes for different use cases. 
The test is guaranteed to be consistent and has tight control over the Type I error rate. 
We evaluate our approach using two case studies: monitoring model changes in (a) toxicity and (b) translation performance. 
We find that the test is able to detect distribution changes in model behavior using hundreds of prompts.
\end{abstract}
% Our test compares model generations from a baseline model to those of the model under scrutiny and provides theoretical guarantees for change detection while controlling false positives.

\section{Introduction}
 
Language models (LMs) can now achieve human-level performance in a wide range of tasks, including text summarization, machine translation, coding, and even acting as AI scientists: generating hypotheses and designing experiments \citep{achiam2023gpt, katz2024gpt, lu2024ai,zhang-etal-2024-benchmarking}. 
As capabilities continue to scale, evaluating LM behaviors becomes increasingly important and increasingly difficult \citep{hendrycks2021unsolved, ngo2022alignment, wolf2023fundamental}.
Large-scale evaluations—such as comprehensive behavior and capability assessments \citep{wang2023decodingtrust} and red-teaming exercises \citep{DBLP:conf/emnlp/PerezHSCRAGMI22}—are widely used to verify that language models (LMs) behave safely and as expected. However, these evaluations tend to be expensive and are not well-suited for continuous monitoring, especially when models are  updated or fine-tuned with new data. This is problematic because even seemingly benign or narrow modifications can inadvertently lead to undesirable changes in model behavior~\citep{qi2023fine,betley2025emergent}. This raises the question: How can we quickly and cheaply detect unwanted changes in LM behavior?
% Because of this, many are looking for ways to use them to improve existing systems, and create new ones  \citep{kasneci2023chatgpt,felten2023will}.
%
% Unfortunately, one large roadblock for LM usage is unexpected behavior that can arise from seemingly benign model modifications \cite{qi2023fine,casper2023open}.
% % propensity of LMs to generate harmful content \citep{weidinger2021ethical}.
% %
% For example, \citet{qi2023fine} found that fine-tuning with benign datasets can inadvertently degrade the safety alignment of LMs. 
% For example, GPT-3 has significant anti-Muslim biases~\citep{abid2021persistent}, and GPT-4 has racial and gender biases~\citep{zack2024assessing}.
%
% The field of \emph{AI alignment} is aimed at ensuring LM behavior is aligned with our societal values~\citep{ji2023ai}.
%
% This has spurred work on LM behavior evaluation through benchmarks~\citep{wang2023decodingtrust} and red-teaming~\citep{DBLP:conf/emnlp/PerezHSCRAGMI22}. At the same time
% Large-scale LM behavior and capability evaluations~\citep{wang2023decodingtrust} and red-teaming~\citep{DBLP:conf/emnlp/PerezHSCRAGMI22} are the methods of choice to assess model behavior; however, these are often expensive and thus not suitable for continuous monitoring. This raises the question: How can we quickly detect unwanted changes in behavior?
%

Consider two hypothetical settings where this question might be asked: \begin{inparaenum}[(1)]
\item \emph{Internal Audit}: A company develops a language model that has passed rigorous safety and performance evaluations. After deploying the model, they continue to fine-tune it to improve its performance on certain tasks. The development team wants to stay informed about any \emph{drastic} changes this might induce in the model's behavior—particularly shifts in areas unrelated to the intended updates. How can the team rapidly detect meaningful changes in model behavior throughout the development cycle?
% PUT HERE THE CITATIONS ABOUT FINETUNING LEADING TO MISALIGNMENT
%
\item \emph{External Audit}: A regulatory body certifies a language model for public deployment after extensive safety evaluation. However, they are concerned that the deployed model's behavior may change over time due to updates or intentional modifications. Since they only have access to the model through an API and cannot inspect its internal parameters, they require a mechanism to regularly check that the model's behavior remains consistent with the certified version. How can the regulator regularly check the deployed model's behavior is the same as the previously certified one?
\end{inparaenum}
We call the problem of detecting changes in LM behavior distributions over time \emph{behavioral shift auditing} problems.

In this paper, we formalize the problem of behavioral shift auditing in language models and propose a statistical test that monitors changes in model behavior using only black-box access (\eg, via API calls). Our goal is to develop a sample-efficient method that guarantees detection of behavioral shifts while tightly controlling the rate of false positives. Further, it should provide the user with a tolerance parameter that allows a behavior distribution to change by some amount $\epsilon$ without triggering a detection. This parameter controls the strictness of the auditing test - in some settings (\eg, example (1)), a more liberal $\epsilon$ might be appropriate, while in other cases (\eg, example (2)) one might require a more conservative $\epsilon$ or even want to disallow any change at all.
The key insight behind our approach is to frame behavior shift auditing as a hypothesis testing problem over the model’s behavior distribution. This framing makes our method applicable to a wide range of measurable behaviors—such as dangerous capabilities \citep{phuong2024evaluating}, mathematical reasoning \citep{mishra2022lila}, and biases \citep{wang2023decodingtrust,kotek2023gender}. To this end, we leverage and extend recent advances in testing by betting  \citep{DBLP:conf/aistats/PandevaFRS24}. Under mild assumptions, our sequential test provably detects any change given enough samples, while ensuring non-asymptotic control over false positives. We demonstrate our test on detecting shifts in toxicity and translation performance. We find that we can detect changed LM behaviors using hundreds of prompts. %serving as an effective early warning system before committing to expensive full-scale evaluations. 
We release our code here: \url{https://github.com/richterleo/lm-auditing-test}.

\begin{figure*}[t]
% \vskip 0.2in
\begin{center}
\centerline{\includegraphics[width=\textwidth]{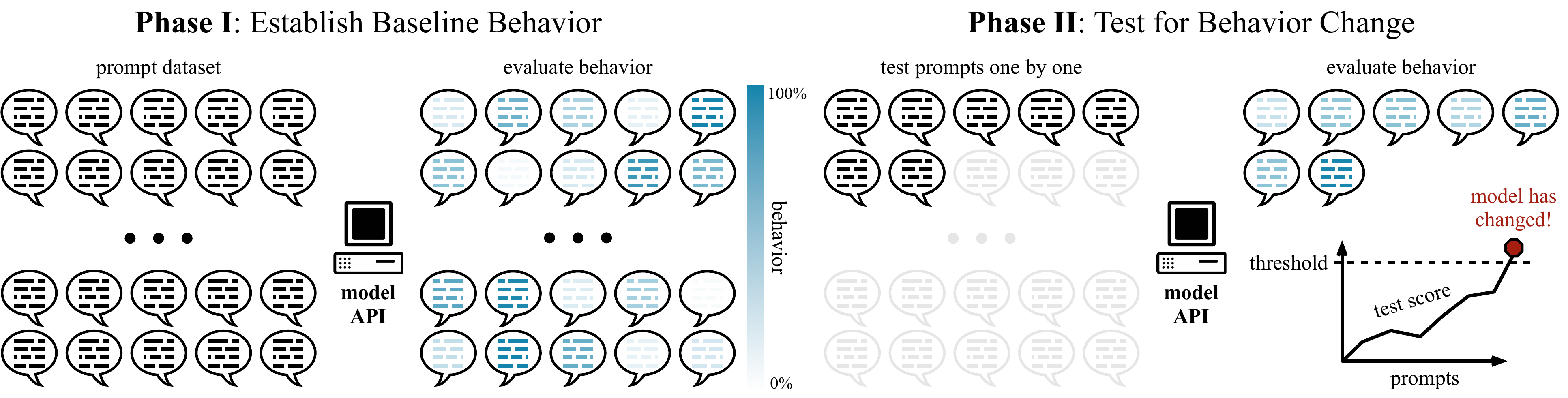}}
\caption{Overview of behavior shift auditing framework.}%\textbf{Behavior shift auditing.} A regulator can use the test we describe to perform an external audit: 1. The regulator initially certifies an LM by prompting and evaluating the set of generations received; 2. Later, tipped off that LM behavior may have changed, the regulator poses as a consumer and sends prompts to the model vendor, collecting the generations; 3. The regulator compares the distribution of behavior scores $b(\cdot)$ between the initial, certified generations and the later generations using a test. If the distributions are sufficiently different the test triggers. Using our proposed method, the regulator can test samples sequentially without increasing the false-positive rate. The method is guaranteed to detect a change if one exists, given enough samples (more details in \cref{sec:methods}).}
\label{fig:overall}
\end{center}
% \vskip -0.2in
\vspace{-6mm}
\end{figure*}

\section{Related Work}
\label{sec:related_work}

\paragraph{LM behavior functions.} 
% benchmarking
% red-teaming
% check out AI Alignment survey paper for references 
Early evaluations of NLP models relied on curated datasets for detecting biases or toxicity \citep{bolukbasi2016man}; larger collections of data were constructed \eg through web scraping \citep{zhao2018gender,zampieri2019predicting,nangia2020crows,rosenthal2021solid} and, more recently, by leveraging LLMs themselves to generate data \citep{zhang2022constructing,perez2023discovering}.  Meanwhile, early work on behavior functions focused on measuring bias, toxicity, and hallucinations \citep{vidgen2020learning, achiam2023gpt, team2023gemini, chern2023factool, varshney2023stitch,llama3}. Since the rise of LMs with human-level performance, the set of behavior functions has exploded~\citep{zou2023universal}. It has become more nuanced, including complex characteristics such as power-seeking behavior \citep{park2023generative, sharmatowards}, situational awareness \citep{zou2023representation}, and deception \citep{hagendorff2024deception}. However, even with access to massive datasets and carefully constructed behavior functions it can be difficult to discover these behaviors from static inputs \citep{kalin2020black}. To address this, \citet{DBLP:conf/emnlp/PerezHSCRAGMI22} introduced the notion of \emph{red-teaming} for LM alignment. This allows prompts to be adversarially-constructed to expose failure cases, which arise in many state-of-the-art models \citep{chaojailbreaking}. %The testing procedure we present here is agnostic to both the prompts and the behavior functions used in LM evaluation.

\paragraph{Model change identification.} For the case where one wishes to identify \emph{any change} in model behavior (\ie, $\epsilon=0$) there are multiple other techniques that can be used. The first set uses ideas from \emph{formal verification} to ensure that the predictions from a model are guaranteed to come from a specific model \citep{ghodsi2017safetynets,dong2021veridl,fan2023validating,weng2023pvcnn}. In general, however, these methods are computationally intensive and do not scale to state-of-the-art LMs. A second, more efficient idea is to \emph{watermark} the model \citep{zhu2018hidden, amrit2022survey, He_Xu_Lyu_Wu_Wang_2022, he2022cater, kirchenbauer2023watermark, kuditipudi2023robust, yoo-etal-2023-robust}. The idea is to embed signals into model generations that can be detected algorithmically. However, watermarks are often inserted by the model owner \citep{kirchenbauer2023watermark,kuditipudi2023robust}, allowing them (or an actor that has compromised the model) to insert it into any model that is being audited. This precludes its use for many external auditing settings. For internal auditing, a watermark may break under a small model change that is acceptable. Our work is also related to work on concept drift \citep{bayram2022concept} and prompt stability \citep{li2024measuring}. In principle our test can be used to detect concept and generation changes, however the focus of these works is on model performance and generation similarity, as opposed to behavior change. 

\paragraph{Sequential hypothesis testing.}
Sequential hypothesis testing allows one to analyze data without fixing the sample size in advance~\citep{wald1945sequential}, offering the potential for greater sample efficiency when significant effects exist~\citep{arrow1949bayes}.
However, naive repeated testing can increase the Type I error rate (i.e., false positives) as the number of tests grows~\citep{jennison1999group}.
To prevent this inflation of false positives, various methods have been developed, including the recent \emph{testing by betting} framework~\citep{robbins1970statistical, ramdas2023game}, which preserves statistical efficiency while tightly controlling the Type I error rate.
Within this framework, a method called Deep Anytime-Valid Testing (DAVT) \citep{DBLP:conf/aistats/PandevaFRS24} designs powerful sequential non-parametric tests by integrating deep learning models into the testing by betting framework. They demonstrate, on a variety of tasks, including two-sample testing, competitive performance compared to other state-of-the-art non-parametric sequential tests, such as the E-C2ST \citep{lheritier2018sequential} and Seq-IT \citep{podkopaev2024sequential}. DAVT uses a model, trained on past observations, to produce an optimized betting score on new data. In this work, we will extend this test to include a tunable tolerance parameter $\epsilon$. %It is also consistent under the assumptions of \citet[Proposition 4.3]{DBLP:conf/aistats/PandevaFRS24}, \ie, the power of the test converges to $1$ as the number of samples goes to infinity.

\section{Preliminaries}

\paragraph{Testing by betting.} 
%
% The fundamental idea behind this framework is the principle of \textit{testing by betting}, inspired by game theory \citep{shafer2021testing}.
%
The testing by betting framework represents evidence against the null hypothesis as the gain in wealth $W$ of a bettor wagering on observed samples \citep{shafer2021testing}.
Before observing new samples, the bettor ``buys'' a test statistic at the ``price'' of its expected value under $ \mathbf{H_0}$.
After new samples are obtained, the bettor's wealth $W$ is multiplied by the ratio between the actual observed test statistic and its expectation. This ratio is referred to as the \textit{betting score} $S_t$. 
The bettor reinvests in subsequent ``rounds'' (\ie, as new data is observed), and the observed betting scores are repeatedly multiplied, leading to a cumulative wealth process.
Under $\mathbf{H_0}$, no betting strategy can consistently increase the bettor's wealth, ensuring control over the Type I error rate \citep{ramdas2023game}.
Let the bettor's (non-negative) wealth after $t$ (batches of) observations be $W_t$. In order to design a test from this wealth process we require that $W_t$ satisfies the following
\begin{equation} \label{eq:eprocess}
    \sup_{P \in \mathbf{H_0}} \mathbb{E}_{P}[W_t] \leq 1 \quad \text{for every } t \geq 0.
    % \label{eq:eprocess}
\end{equation}
All non-negative stochastic processes $W_t$ that satisfy the above condition are called an \textbf{e-process} for $\mathbf{H_0}$ \citep{howard2021time}.
This states that the maximum wealth across all bets cannot exceed $1$ if the null hypothesis $\mathbf{H_0}$ is true.\footnote{It can be shown that the wealth process $W_t$ defined this way is equivalent to the minimum wealth a bettor can obtain across all $P \in \mathbf{H_0}$ \citep{ramdas2023game}.}
% The e-process represents the maximum wealth across all bets under $\mathbf{H_0}$ carried out simultaneously and can be interpreted as the cumulative evidence against the null hypothesis.
%
% A \textbf{stopping time} \( \tau \) with respect to the filtration \( \{\mathcal{F}_t\} \) specifies the time at which the test stops based on the observed data. Formally, \( \tau \) is a random variable taking values in \( \{0, 1, 2, \dots\} \) such that the event \( \{\tau = t\} \) is \( \mathcal{F}_t \)-measurable for each \( t \geq 0 \) \citep{doob1953stochastic}. The stopping time can be any rule that depends only on past and present observations.
%
Given an e-process, the test is constructed as follows: reject the null $\mathbf{H_0}$ at some time $\tau$ if $W_\tau \geq \gamma$, where $\gamma = \alpha^{-1}$ is a threshold defined by a desired significance level \( \alpha \in (0,1) \).
% The tests \textbf{decision rule} is to reject $\mathbf{H_0}$ at a stopping time $\tau$ if the e-process $W_t$ crosses a threshold \( \gamma = \alpha^{-1} \), where \( \alpha \in (0,1) \) is the desired significance level:
% %
% \begin{align*}
%     W_\tau \geq \gamma.
% \end{align*}
%
Under \( \mathbf{H_0} \), the e-process \( W_t \) controls the Type I error rate. By Ville's inequality~\citep{ville1939etude}, we have:
\begin{equation}
    \mathbb{P}_{\mathbf{H_0}} \left( \sup_{t \geq 0} W_t \geq \gamma \right) \leq \frac{1}{\gamma} = \alpha.
    \label{eq:ville}
\end{equation}
This ensures that the probability of incorrectly rejecting $\mathbf{H_0}$ is at most $\alpha$ at any time step. 
Thus, the sequential test is \emph{anytime-valid}, maintaining error control at any stopping point.% no matter when it is stopped.

\section{Detecting Behavior Changes}
\label{sec:methods}

% The goal of Behavior Shift Auditing (BSA) is to detect changes in LM behavior over time. 
% To motivate Behavior Shift Auditing (BSA), we detail an external auditing example in Figure~\ref{fig:overall}. 
We propose an anytime-valid test for behavior shift auditing that has guarantees on its false positive rate and is consistent under certain weak assumptions. Building upon the two-sample variant of DAVT~\citep{DBLP:conf/aistats/PandevaFRS24}, our test introduces a customizable tolerance parameter $\epsilon$ that allows users to
specify what constitutes a practically significant difference between distributions, accommodating
small, insignificant variations. While we concentrate here on its application for behavior shift auditing, it may be of independent interest to the sequential hypothesis testing community. We describe the test in full generality in  Appendix~\ref{sec:derivation_and_proofs}. 
% Building upon the two-sample variant of DAVT \citep{DBLP:conf/aistats/PandevaFRS24}, our test introduces a customizable tolerance parameter $\epsilon$ that allows users to specify what constitutes a practically significant difference between distributions, accommodating small, insignificant variations. This approach diverges from prior sequential tests that check for exact distribution equality \citep{ramdas2023game, shekhar2023nonparametric, DBLP:conf/aistats/PandevaFRS24}, which may be overly sensitive for our use cases. We describe the test in full generality in  Appendix~\ref{sec:derivation_and_proofs}. T
% o focus the text and avoid notational complexity, we concentrate here on the application to behavioral shift detection in LMs.

\subsection{Auditing Test}
\label{subsec:auditing_test}

% Let $B(\rvx, \rvy)$ be such a behavior function that takes as input a prompt $\rvx \in \mathcal{X}$ and a generation $\rvy \in \mathcal{Y}$.\footnote{We include the prompt for generality, there is no requirement that $B$ must depend on the prompt.} For simplicity, assume that $B$ assigns scores in the range $[0,1]$ where $1$ represents the full manifestation of the behavior and $0$ indicates its absence \citep{perez2023discovering,wolf2024fundamental}. We use the concept of LM behavior generically, including capabilities and performance on tasks.\footnote{Our proposed method can also be applied to detect \textit{concept drift} (see \citet{bayram2022concept} for a review).}  A language model is a stochastic operator $M$ that maps prompts $\rvx$ to generations $\rvy$.
Let $X$ be a random variable representing a prompt, $\mathcal{X}$ the set of possible prompts, and $\rvx \in \mathcal{X}$ a realization of $X$. A language model is a stochastic operator $M$ that maps prompts $\rvx$ to generations $\rvy$.
A behavior scoring function $B$ is a stochastic operator that takes a prompt and generation as input\footnote{We include the prompt for generality, there is no requirement that $B$ must depend on the prompt.} and produces a score $B(\rvx, \rvy) \in [0,1]$ \citep{perez2023discovering,wolf2024fundamental}.  % (this is without loss of generality, for any dimension $d$).
The behavior function, prompts, and language model induce a \emph{behavior distribution} $P_{B}^{\baselineM}$ over behavior scores  $B(X, \baselineM(X))$. We can now frame the question of whether the behavior of a model $\changedM$ has changed (substantially) relative to a baseline model $\baselineM$ as a testing problem:
\begin{align}
    \mathbf{H_0}:\; \mathcal{D}\big(\baselineDist, \changedDist\big) \leq \epsilon \quad \text{vs.} \quad 
    \mathbf{H_1}:\; \mathcal{D}\big(\baselineDist, \changedDist\big) > \epsilon,
    \label{eq:tolerance_h0}
\end{align}
where $\epsilon \geq 0$ is a tolerance parameter, and $\mathcal{D}$ is a distance measure between probability distributions. Note that equality in the null hypothesis in eq.~(\ref{eq:tolerance_h0}) corresponds to DAVT \citep{DBLP:conf/aistats/PandevaFRS24}. To extend this to the composite case, our goal is to \emph{construct an appropriate wealth process $W_t$}. This will allow us to establish error rate and consistency guarantees. To do so, we will define a betting score $S_t$ such that it produces a wealth process $W_t$ that is an e-process \ie, it satisfies eq.~(\ref{eq:eprocess}). This, in turn, will depend on the distance measure $\mathcal{D}$ that we choose.

Given a batch of prompts $x_1, \dots, x_b$ and the distance threshold \( \epsilon \) from \cref{eq:tolerance_h0}, we propose the \textit{betting score}
\begin{align}
    S_t = \prod_{i=1}^b \left( \frac{1 + \phi_{t-1}\big( B(x_i, \baselineM(x_i)) \big) - \phi_{t-1}\big( B(x_i, \changedM(x_i)) \big) }{ \exp(\epsilon) } \right).
    \label{eq:empirical_s}
\end{align}
where $\phi_{t-1}$ is a neural network trained on all $(t\!-\!1)$ previous batches to optimize the objective 
\begin{align*}
\max_\phi \mathbb{E}[\log \left(1+ \phi(B(X, \baselineM(X)))-\phi(B(X, \changedM(X))\right)].
\end{align*}
Given the betting score $S_t$, we define the \textit{wealth process} $\{W_t\}_{t \geq 1}$ of a bettor by initializing their wealth as $W_0 = 1$ and updating
\begin{align}
    W_t = W_{t-1} \times S_t.
    \label{eq:wealth_update}
\end{align}

If the betting score $S_t$ is an e-variable, meaning that $\mathbb{E}_\mathbf{H_0}[S_t] \leq 1$, then the wealth process $\{W_t\}_{t \geq 0}$ is an e-process, which we can prove by induction. Under $\mathbf{H_0}$, and for any fixed $\baselineDist, \changedDist$ satisfying $\mathcal{D}_\Phi(\baselineDist, \changedDist) \leq \epsilon$, $W_{t-1}$ and $S_t$ are independent. Therefore,
\begin{align*}
    \mathbb{E}_{\mathbf{H_0}}[W_t] &= \mathbb{E}_{\mathbf{H_0}}[W_{t-1} \times S_t] \\
    &= \mathbb{E}_{\mathbf{H_0}}[W_{t-1}] \times \mathbb{E}_{\mathbf{H_0}}[S_t] \leq \mathbb{E}_{\mathbf{H_0}}[W_{t-1}],
\end{align*}
% since, by assumption, $S_t$ is an e-variable we have that \( \mathbb{E}_{\mathbf{H_0}}[S_t] \leq 1 \). 
By induction, \( \mathbb{E}_{\mathbf{H_0}}[W_t] \leq 1 \) for all \( t \geq 0 \).

To ensure that $S_t$ is indeed an e-variable, we choose an appropriate distance measure in eq.~(\ref{eq:tolerance_h0}). Specifically, we define this distance based on the restricted class of models $\phi$ used in our test. As in \citep{DBLP:conf/aistats/PandevaFRS24}, we make the following assumptions on $\phi$: 
\begin{assumption}[\citet{DBLP:conf/aistats/PandevaFRS24}]%[Neural Net Distance]
\label{a:model}
The model class used in our test $\Phi = \{\phi_\theta: \theta \in \Theta\}$ must satisfy the following properties:
\begin{itemize}
    \item For all $\phi \in \Phi$ and for all $s \in [0,1]$, \: $|\phi(s)| \leq q$ for some $q \in (0, 1/2)$.
    \item If $\phi \in \Phi$, then $c \cdot \phi \in \Phi$ for every $c \in [-1, 1]$.
\end{itemize}
\end{assumption}

We can now define the distance measure used in our test.

\begin{definition}[Neural Net Distance]
Define the distance\footnote{This distance is an instance of an \emph{integral probability metric} (IPM) \citep{muller1997integral}, a class of distances that includes well-known metrics like the Wasserstein distance \citep{kantorovichspace}. IPMs are at least pseudo-metrics \ie, they satisfy all the properties of a metric except that the distance between distinct points can be zero.} used in eq.~(\ref{eq:tolerance_h0}) to be
\begin{align}
    \mathcal{D}_{\Phi}\big(\baselineDist, \changedDist\big) 
    = \sup_{\phi \in \Phi} \mathbb{E}\left[\phi(B(X, \baselineM(X)) - \phi(B(X, \changedM(X))\right].
    \label{eq:neural_net_distance}
\end{align}
\label{definition:nn_distance}
\end{definition}
For this distance, $S_t$ is an e-variable (see Appendix~\ref{app:practicaltest} for a proof). We can now define the following \textbf{sequential test}
\begin{align}
    \gamma = \inf \left\{ t \geq 1 : W_t \geq \frac{1}{\alpha} \right\}.
    \label{eq:practical_seq_test}
\end{align}

Control over the Type I error follows again from Ville's inequality~(\ref{eq:ville}). The test is consistent under the following assumptions. 

\begin{proposition}
    If the learning algorithm satisfies the condition 
    \begin{align}
        \liminf_{t \rightarrow \infty} \frac{ \mathbb{E} \left[ \log\left( \frac{1}{\exp(\epsilon)} \left( 1 + \phi_{\theta_t}(X_t) - \phi_{\theta_t}(Y_t) \right) \right) \mid \mathcal{F}_{t-1} \right] }{ 3c \sqrt{ \log(t)/t } } \overset{\text{a.s.}}{\geq} 1
        \label{eq:assumption_on_learning_algo}
    \end{align}
    for all $\baselineDist, \changedDist$ with $\mathcal{D}_\Phi(\baselineDist, \changedDist) > \epsilon$ and for a universal constant $c$, then we have 
    \begin{align}
        P_\mathbf{H_0}(\gamma < \infty) \leq \alpha \quad \text{and} \quad P_\mathbf{H_1}(\gamma < \infty) = 1
    \end{align}
    % \begin{align}
    %     P(\gamma < \infty) =1 \quad \text{for all } P_1, P_2 \text{ with } \mathcal{D}_\Phi(P_1, P_2)>\epsilon.
    % \end{align}
    \label{prop:practical_test_consistency}
\end{proposition}
For the proof, see Appendix~\ref{app:proof_practical_test}. This sequential test is thus a \textit{sequential level-\( \alpha \) test of power one}.

% \begin{wrapfigure}{R}{0.5\textwidth}
% \resizebox{0.5\textwidth}{!}{
% \begin{minipage}[H]{0.49\textwidth}
% \centering
%     % \begin{subfigure}[b]{0.48\textwidth}
%         \includegraphics[width=\textwidth]{Plots_new/power_llama.pdf}
%         %\caption{x}
%     % \end{subfigure}
%     %\hspace{0.04\textwidth}
%     % \begin{subfigure}[b]{0.49\textwidth}
%         % \includegraphics[width=\textwidth]{Plots_new/power_over_wasserstein_distance_grouped_by_fold_size_Meta-Llama-3-8B-Instruct_seed1000.pdf}
%         %\caption{y}
%     % \end{subfigure}
%     \caption{\textbf{Fine-tuning Detection for Llama-3-8B-Instruct.} The detection frequency as a function of number of generated samples. Each curve is a fine-tuned corrupted model checkpoint (to simplify visualization, the curves with shaded standard deviations are averages over models with similar distances to the aligned model). The color depicts the Wasserstein distance between the corrupted model and the original aligned model.} %(\emph{Left})  (\emph{Right}) Detection frequency as a function of distance to the aligned model. Each point represents a corrupted model, lines are colored based on the number of observed generated samples.}
%     \label{fig:llama3_power}
% % \end{figure*}
% \end{minipage}
% }
% \vspace{-2ex}
% \end{wrapfigure}

\begin{wrapfigure}{r}{0.5\textwidth}
\vspace{-1ex}
\resizebox{0.5\textwidth}{!}{
\begin{minipage}[b]{0.6\textwidth}
\begin{algorithm}[H]
   \caption{Auditing Test}
   \label{alg:auditing_test}
      \begin{algorithmic}[1]
        \STATE \textbf{Input:} $\{\rvx_t\}_{t \geq 1}$ (stream of prompts), $B$ (behavior function),  $\baselineM$ (baseline model API), $\changedM$ (current model API), $\alpha$ (type-I error limit under null),  $\phi_0$ (neural net model for testing), $\epsilon$ (maximal neural net distance) % $T$ (maximum number of observations), 
        % $W_0$ (wealth) %, $\gB_0$ (initial behavior data)
        \STATE $W_0 \leftarrow 1$ %, \gB_0 \leftarrow \emptyset$
        % \FOR {$t$\ $\leftarrow$ 1 to $T$}
        \WHILE{true}
            \STATE Compute behavior scores: \\ 
            \;\; $b_t \leftarrow B(\rvx_t, \baselineM(\rvx_t)), b_t' \leftarrow B(\rvx_t, \changedM(\rvx_t))$ \label{line1}
            \STATE Compute betting score: \\
            \;\; $S_t \leftarrow \frac{(1 + \phi_{t-1}(b_t) - \phi_{t-1}(b_t'))}{\exp(\epsilon)}$ 
            \STATE Update wealth: \\
            \;\; $W_t \leftarrow W_{t-1} \times S_t$
            \IF {$W_t \geq 1/\alpha$}
                \STATE Break and reject null
            \ENDIF
            % \STATE Add to behavior data: \\
            % \;\; $\gB_t \leftarrow \gB_{t-1} \cup (B^f_t, B^c_t)$
            \STATE Update neural net model: \\
            \; \; $\phi_t \leftarrow \argmax_\phi \sum_{l=1}^t \log(1 + \phi(b_t) - \phi(b_t'))$ \label{lst:line:update}
            % \STATE $\alpha_t \leftarrow (1 - \bar{\alpha}) \exp(-\gamma \cdot t) + \bar{\alpha}$
            % \STATE $p_d \leftarrow \frac{1 - \alpha_t}{L}$ \COMMENT{Layer decay.}
            % \FOR {$i$\ $\leftarrow$ 0 to $L-1$}
            %     % \Comment{Only showing the forward pass}
            %     \STATE $s \sim \operatorname{Bernoulli}(p)$ \COMMENT{Keep or drop.}
            %     % \State $h \leftarrow gt\_block(h)$
            %     \IF {$s == 0$} 
            %         \STATE $\rvx_t_{i+1} \leftarrow \rvx_t_i$ \COMMENT{Drop.}
            %     \ELSE 
            %         \STATE $\rvx_t_{i}' \leftarrow \rvx_t_{i} + \frac{f_{ATTN}(f_{LN}(\rvx_t_{i}))}{p}$ % \times \frac{1}{p}$
            %         \STATE $\rvx_t_{i+1} \leftarrow \rvx_t_{i}' + \frac{f_{FFN}(f_{LN}(\rvx_t_{i}'))}{p}$ %\times \frac{1}{p}$
            %     \ENDIF
            %     %\STATE $\rvx_t_{i} \leftarrow x_{i+1}$
            %     \STATE $p \leftarrow p - p_d$ \COMMENT{Decay prob.}
            % \ENDFOR
            % \STATE $\ell \leftarrow \mathcal{L}(f_O(\rvx_t_L), \mathbf{y})$
            % \STATE $f_{ATTN}, f_{LN}, f_{FFN}, f_O \leftarrow \operatorname{Update}(\ell)$
        \ENDWHILE
      \end{algorithmic}
     \label{alg:dropping}
     \end{algorithm}
%      \end{minipage}
% }
% \vspace{-5ex}
% \end{wrapfigure}
     \end{minipage}
}
\vspace{-5ex}
\end{wrapfigure}

\subsection{Algorithm}
The auditing test (shown in Algorithm~\ref{alg:auditing_test}) takes in a stream of prompts $\{\rvx_t\}_{t\geq1}$, a behavior function $B$, an initial baseline language model $\baselineM$, a second language model $\changedM$, the $\alpha$-level, a neural net model initialization $\phi_0$, and a tolerance parameter $\epsilon$, representing the maximal neural net distance we want to accept between behavior distributions. At every time step, a new prompt from the stream $\rvx_t$ is fed to both $\baselineM$ and $\changedM$ to create generations, which are then scored by the behavior function. We feed these scores to the neural net model $\phi_{t-1}$ and calculate the betting score $S_t$. Next, we update the wealth $W_t$ by the betting score and check whether it surpasses the $1/\alpha$-threshold, in which case we reject the null hypothesis. If not, we update the neural net model in a separate training step and continue with the next prompt. The algorithm can easily be modified to accept batches instead of single prompts.\footnote{In this case, the new betting score $S_t$ is calculated as a product over samples in the batch.}

\section{Experiments}
\label{sec:expr}
We evaluate our test for both the external and internal auditing use-cases.  We first look at the strict case, where any behavioral change is prohibited, and then move on to the case where small changes in distribution are allowed. We investigate toxicity and translation performance. 

\subsection{Exact test, \texorpdfstring{$\epsilon=0$}{ε=0}}
\label{subsec:exact_text}

\begin{wrapfigure}{R}{0.5\textwidth}
    \centering
    \includegraphics[width=0.45\textwidth]{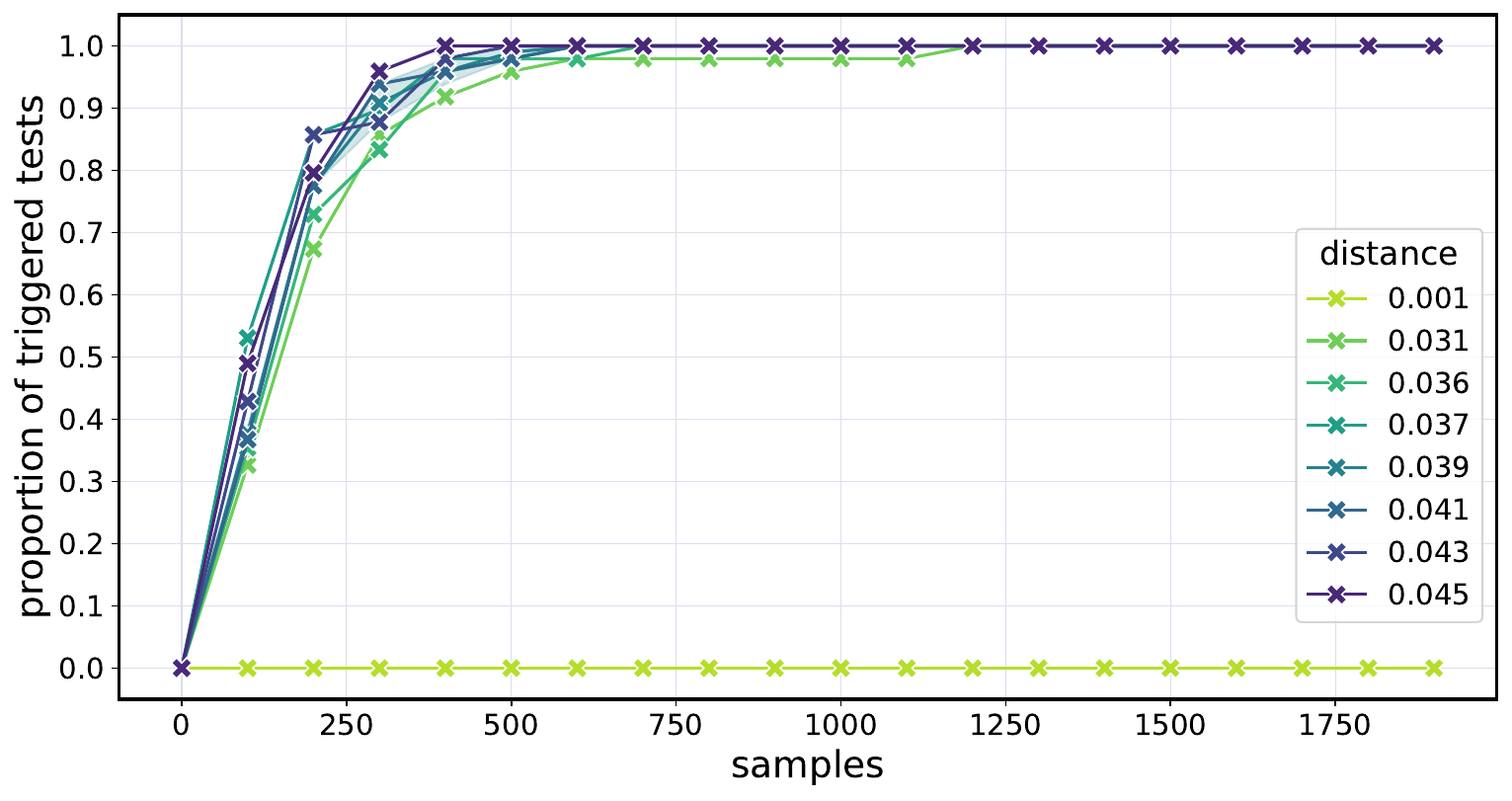}
    \caption{\textbf{Fine-tuning Detection for Llama3-8B-Instruct.} The detection frequency as a function of number of generated samples. Each curve is a fine-tuned corrupted model checkpoint (to simplify visualization, the curves with shaded standard deviations are averages over models with similar distances to the aligned model). The color depicts the Wasserstein distance between the corrupted model and the original aligned model.}
    \label{fig:llama3_power}
    \vspace{-1ex}
\end{wrapfigure}

\paragraph{Setup.}
We begin by investigating an external setting where we require the test to detect any change in distribution ($\epsilon\!=\!0$). Specifically, we will check for changes in toxicity behavior. We select prompts from the REALTOXICITYPROMPTS dataset \citep{gehman2020realtoxicityprompts} and use the toxicity behavior function from Perspective API \citep{lees2022new} to evaluate LM generations. \llama (\texttt{8B-Instruct})~\citep{llama3}, \gemma (\texttt{1.1-7b-it})~\citep{team2024gemma}, and \mistral (\texttt{7B-Instruct-v0.2})~\citep{jiang2023mistral} serve as our initial aligned models. We remove the safety alignment in these models by fine-tuning, producing 10 corrupted checkpoints for each model. To evaluate the statistical properties of our the exact test ($\epsilon\!=\!0$), we assess (a) its ability to detect changed checkpoints, and (b) its false positive rate. 
For further experimental details regarding toxicity fine-tuning, text generation and the betting score network, please see Appendix~\ref{app:subsec:setup}.

\paragraph{Corrupted model detection.} 
We test each corrupted checkpoint against the corresponding initial aligned model with $\alpha\!=\!0.05$. 
\Figref{fig:llama3_power} shows the fraction of positive test results after having observed at least $m$ samples, with tests repeated 48 times per checkpoint (2000 samples per fold, batch size 100). High detection rates of almost $80\%$ are achieved even for checkpoints closest to the baseline. 
We find that as the distance between the corrupted model and the initial model increases, fewer samples are needed to detect the change in behavior. 
Similar results for Mistral and Gemma can be found in Appendix~\ref{sec:further_results}.

\paragraph{False positive rate.}
We use different random seeds for generating text from the initial aligned models to examine the false positive rate of the exact test. Figure~\ref{fig:false_positives} shows the false positive rate for each of the model architectures as a function of the number of observed samples, repeated 24 times (4000 samples per fold, batch size 100). The test is highly specific, with false detection rates consistently below $0.05$.

\subsection{Tolerance Test, \texorpdfstring{$\epsilon>0$}{ε>0}}
 
We now evaluate the test with tolerance $\epsilon > 0$ in two use-cases: an external toxicity audit, and an internal translation performance audit. 

In both cases, the exact test might be too sensitive. However, how much variation to allow between distributions might depend on the use-case. We thus want to explore some possible strategies for determining the hyperparameter $\epsilon$ appropriately in each scenario.

\begin{figure}[t!]
    \centering
    \includegraphics[width=\linewidth]{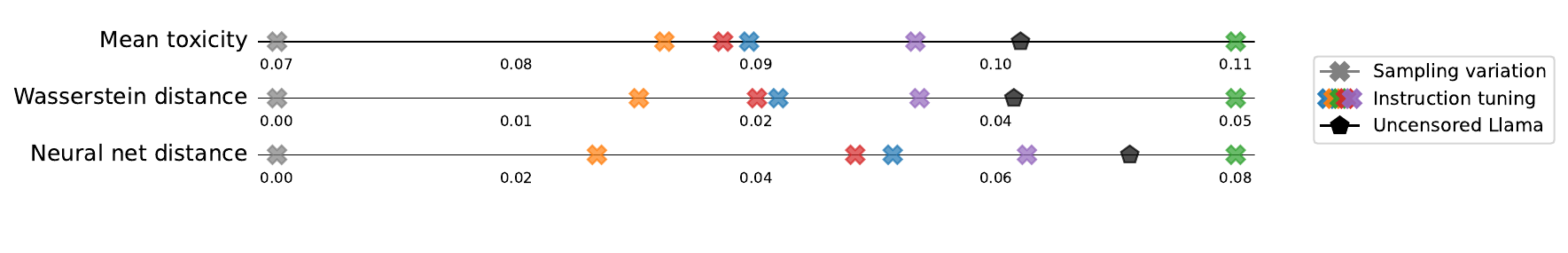}
    % \caption{\textbf{Measuring Mean and Distributional Change}. Toxicity metrics for Llama3 variants: baseline with varied sampling, five instruction-tuned models, and an uncensored model. We observe mean toxicity increases in all instruction-tuned models, with one even surpassing the uncensored model. Wasserstein and neural net distances to the baseline show corresponding increases, indicating consistent toxicity shifts across measures.}
    \caption{\textbf{Measuring Mean and Distributional Change}. Analysis of seven Llama3-8B variants shows aligned shifts across three metrics: mean toxicity scores, Wasserstein distances, and Neural net distances to baseline Llama3-8B-Instruct. The variants include the baseline model with modified sampling parameters, five models instruction-tuned on subsets of SuperNI, and an uncensored model.}

    \label{fig:means_wasserstein_nn}
\end{figure}

\subsubsection*{Use Case 1: External Audit, Toxicity}
\label{subsubsec:usecase1}

\begin{wrapfigure}{R}{0.5\textwidth}
% \vspace{-2ex}
\resizebox{0.5\textwidth}{!}{
\begin{minipage}[H]{0.49\textwidth}
\centering
    \includegraphics[width=\textwidth]{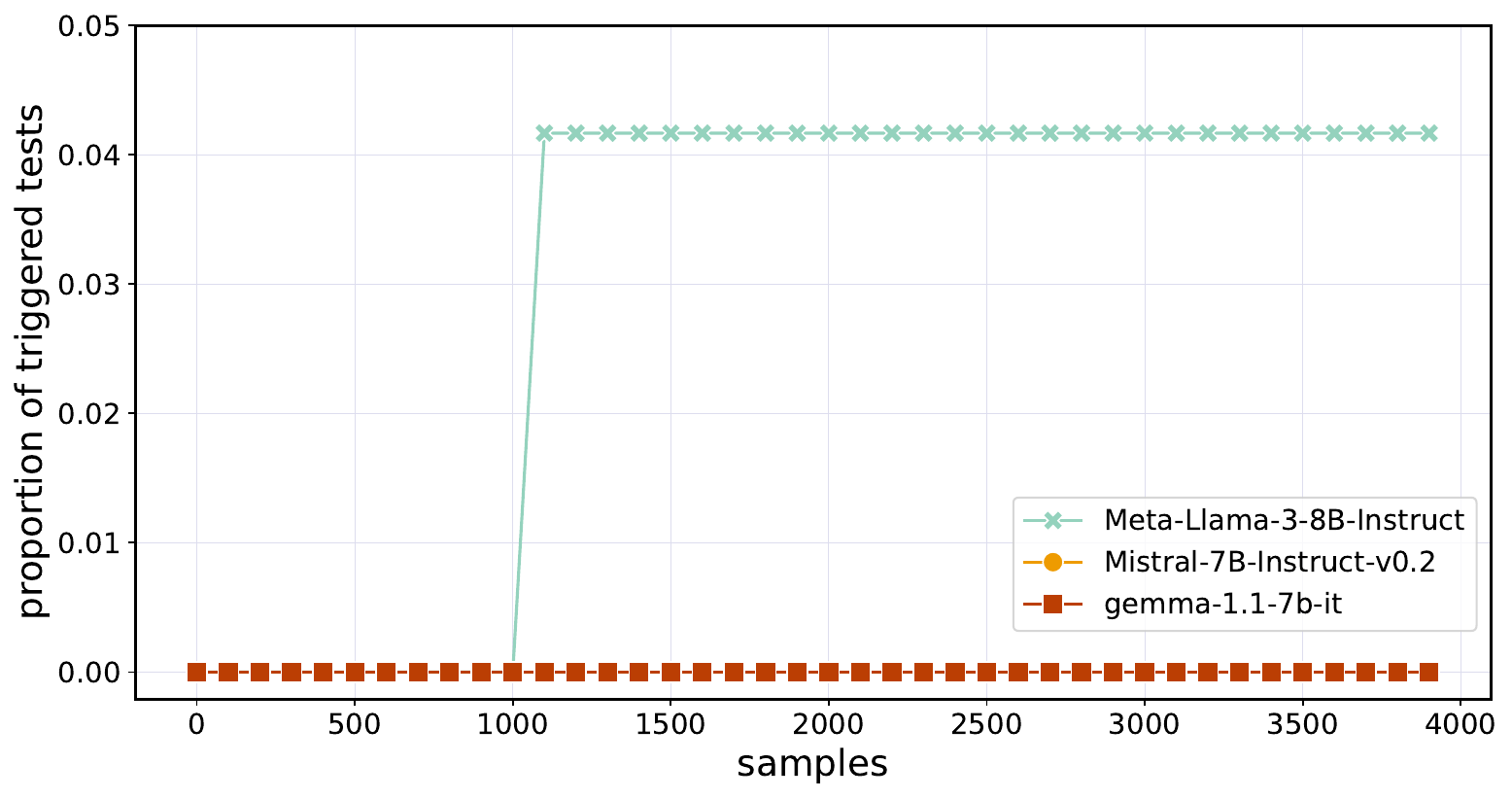}
    \caption{\textbf{False positives.} The false positive rate for each of the baseline models as a function of number of observed samples. Using the same model and sampling strategy but different random seeds, we generate two outputs for each prompt to be used as the sample pairs for our auditing test.}
    \label{fig:false_positives}
\end{minipage}
}
\vspace{-2ex}
\end{wrapfigure}

\paragraph{Setup.}
We simulate an external auditor checking whether instruction-tuning an aligned model on unrelated tasks affected toxicity distributions, something that has been observed in practice \citep{qi2023fine}. We use Llama3 (\texttt{8B-Instruct}) as the aligned model, again evaluating toxicity on the REALTOXICITYPROMPTS dataset \citep{gehman2020realtoxicityprompts} using Perspective API \citep{lees2022new}. We instruction-tune \llama on 5 different task clusters from SUPER-NATURALINSTRUCTIONS \citep[SuperNI;][]{naturalinstructions, supernaturalinstructions}. This setup is inspired by \cite{wang2023far}, who found that a pre-trained Llama2 model instruction-tuned on SuperNI exhibits high toxicity scores on ToxiGen. Detailed information on instruction-tuning and how the neural net distance is estimated can be found in Appendix~\ref{app:experimental_details}.

\paragraph{Results.} 
% \label{subsubsec:detectingtoxicity}
Instruction-tuning increased mean toxicity scores, which, as shown in Figure \ref{fig:means_wasserstein_nn}, corresponds with increases in both Wasserstein distances and neural net distances from Llama3. As a reference, we also include another Llama3-8B model tuned to be less refusing.\footnote{The uncensored model was fine-tuned on Uncensored-Vortex \url{https://huggingface.co/datasets/OEvortex/uncensored-vortex}.} Surprisingly, the most toxic and distant model is not this uncensored model but the model fine-tuned on Code to Text, Stereotype Detection, and Sentence Perturbation (shown in \textcolor{ForestGreen}{green}). We test Llama3 against each instruction-tuned model across a range of tolerance values, from $\epsilon\!=\!0.0038$ (the neural net distance between standard Llama3 and Llama3 with different sampling parameters) up to the neural net distance between the base model and another Llama3-8B model tuned to be less refusing, $\epsilon\!=\!0.076$.

Figure \ref{fig:detection_vs_epsilon_finetuned_models} shows the proportion of tests where the fine-tuned model was identified as different from the baseline across various test epsilon values, with tests being repeated 24 times using 4000 samples each. At lower epsilon values, representing a conservative testing regime that detects even small changes, all instruction-tuned models are consistently identified (100\% detection rate). As epsilon increases, the power of the test decreases until it reaches the true neural net distance between the base model and each fine-tuned variant. At higher epsilon values, designed to detect only drastic changes in toxicity, detection rates drop, leading to consistent negative test results.

We investigate the strict auditing setting -- where only minor variations due to sampling are accepted -- more closely. Specifically, we set $\epsilon$ equal to the neural net distance between the original \llama model and the same model with different sampling parameters ($\epsilon=0.0038$) and test baseline \llama against the 5 instruction-tuned versions as well as the uncensored reference \llama. Figure \ref{fig:finetuned_models_detection} demonstrates that under this strict threshold, the test requires fewer samples to detect models that deviate more substantially from the baseline.

\begin{figure}[t!]
    \begin{center}
        \centerline{\includegraphics[width=\columnwidth]{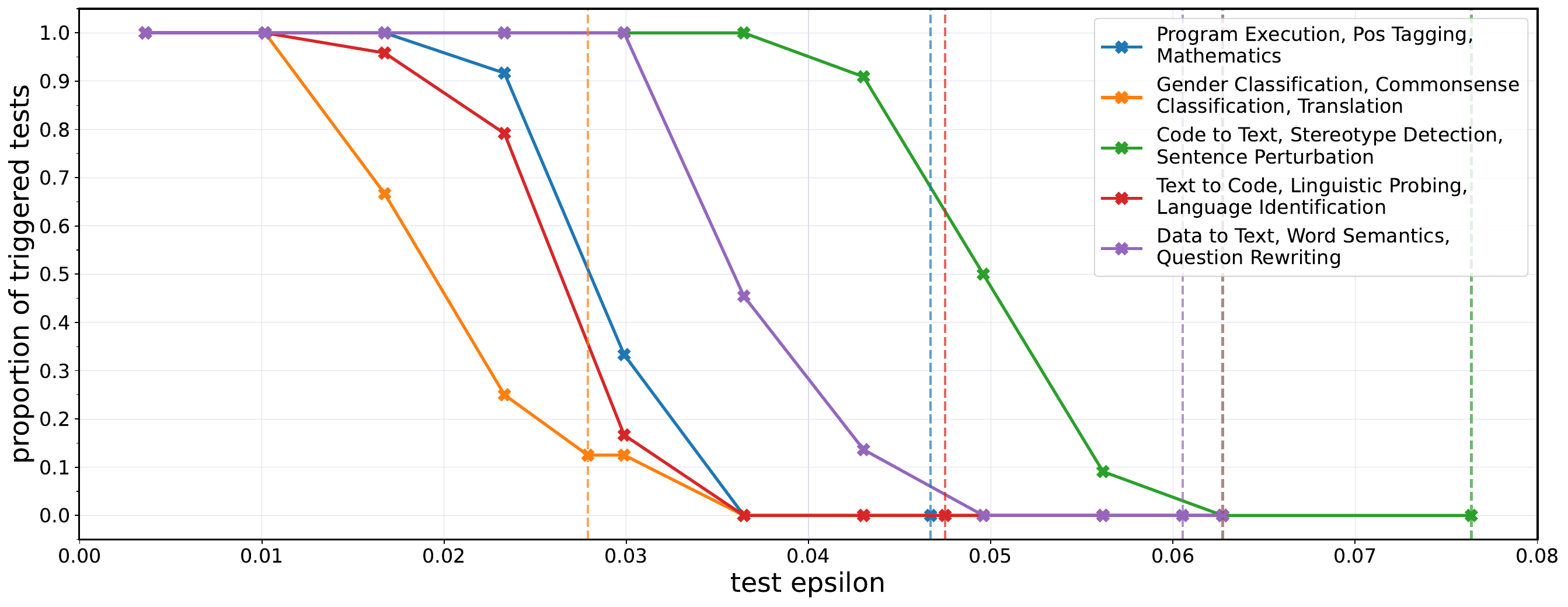}}
        \caption{\textbf{Detection rate vs. Test Epsilon.} Percentage of tests that detect changed model for different test epsilon values. Dashed lines represent estimated true neural net distance between Llama3-8B-Instruct and the instruction-tuned model. We note that the false positive rate for the model fine-tuned on Gender Classification, Commonsense Classification and Translation exceeds the $\alpha$-level of 5\% in two cases, corresponding to 3/24 tests wrongly showing positive results. Assuming a perfect estimate of the true neural net distance, this event can occur with a maximum probability of $8.6$\%.}
        \label{fig:detection_vs_epsilon_finetuned_models}
    \end{center}
    \vspace{-9mm}
\end{figure}

\subsubsection*{Use Case 2: Internal Audit, Translation Performance}
\label{subsubsec:internal}

We simulate a modeler adjusting their language model while monitoring whether its translation capabilities change substantially. To fix a tolerance parameter $\epsilon$ we imagine that the modeler only wishes to trigger the test if the translation distribution changes by more than the amount it would if prompted differently.
%We define ``substantial'' difference using the neural net distance induced by different prompting techniques, %assuming that a difference induced just from prompt engineering could be seen as negligible. 

\begin{wrapfigure}{R}{0.5\textwidth}
\vspace{-2ex}
\resizebox{0.5\textwidth}{!}{
\begin{minipage}[H]{0.49\textwidth}
\centering
    \includegraphics[width=\textwidth]{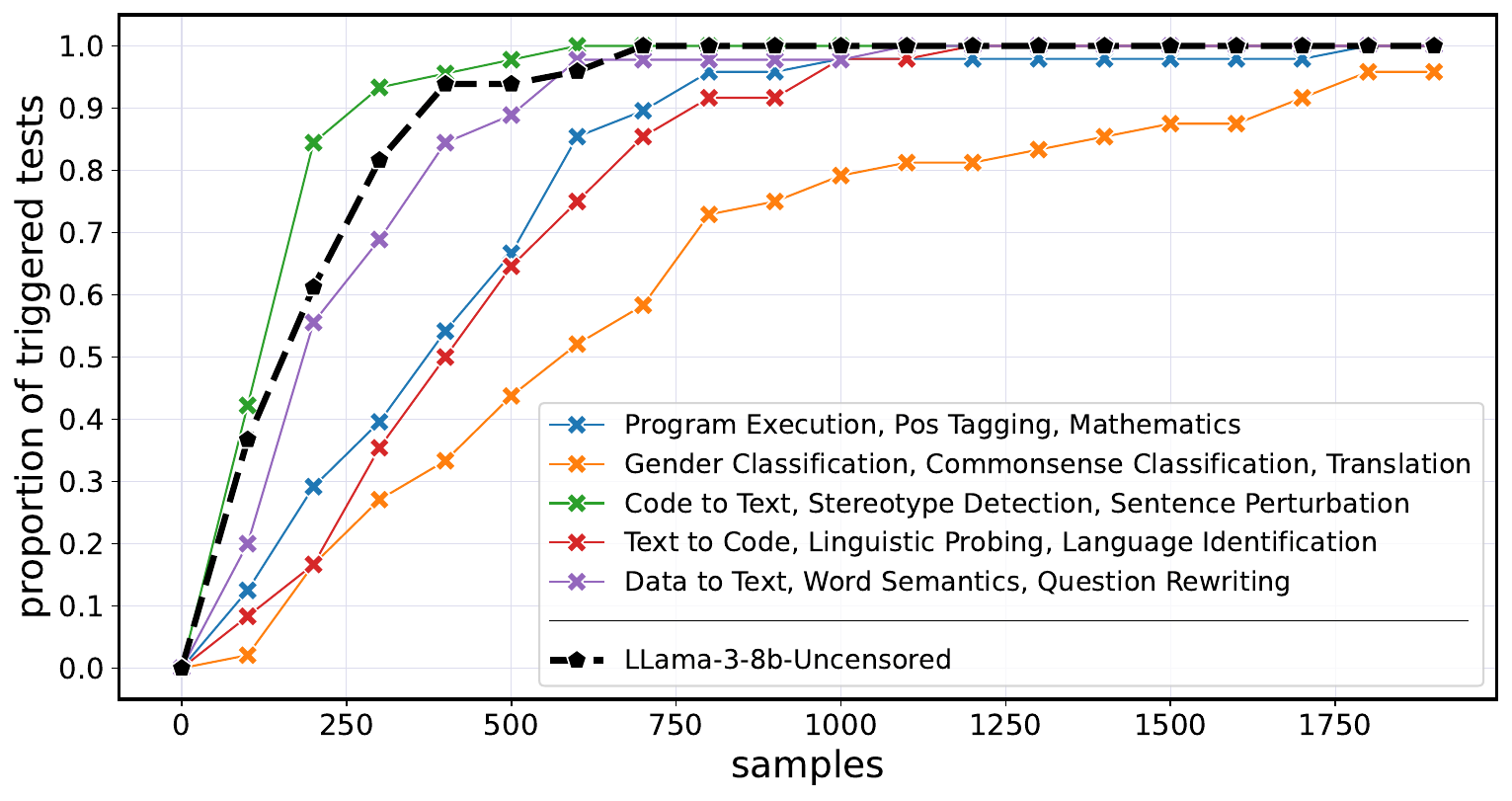}
    \caption{\textbf{Detection Rates for Fine-Tuned Models.} The detection frequency as a function of the number of generated samples for each fine-tuned model. We used a test with $\epsilon \approx 0.0038$, based on the estimated neural net distance between distributions generated by Llama3-8B-Instruct using different sampling parameters. The black line represents an unaligned reference model, Llama3-8B trained to be more permissive in answering.}
    \label{fig:finetuned_models_detection}
\end{minipage}
}
\vspace{-4ex}
\end{wrapfigure}

\paragraph{Setup.}
We evaluate Llama3 (\texttt{8B-Instruct}) on English-Spanish and English-French translations from SuperNI. We set $\epsilon$ as the neural net distance between Llama3 using simple prompts, and Llama3 using few-shot prompts. We then test the translation performance distribution of Llama3 with simple prompts against that of Aya-23-8B~\citep{ustun2024aya}, a multilingual instruction-tuned model. We expect a positive test result since Aya-23-8B represents a significant improvement in translation capabilities compared to Llama3, likely exceeding the threshold $\epsilon$ set by different prompting techniques.

% \paragraph{Internal Audit: Translation}

% We evaluate Llama3 both with and without few-shot prompting as well as Aya-23-8b on a subset of English-French/French-English and English-Spanish/Spanish-English samples. 
\paragraph{Results.}
Few-shot prompting leads to a modest increase in mean BLEU scores from $0.1683$ to $0.1765$. A significant improvement is evident when using Aya-23-8b, with a mean BLEU score of $0.2970$. We observe that Llama3 models occasionally misinterpret instructions or include unnecessary additional text in English, potentially impacting their scores. We run our test comparing simple-prompted Llama3 with Aya-23-8b and report the results averaged over 32 runs in Figure~\ref{fig:translation_aya}. The test detects a difference in nearly all cases after only $100$ samples.
% To assess whether BLEU distributions change significantly, we apply the BSA. Setting $\epsilon=0$ between the standard \llama and the few-shot prompted \llama, the test finds a positive result 22.73\% of the time. Using the distance between these two models as the test epsilon between \llama and Aya-23, BSA detects a difference with almost 100\% probability after observing just two batches of 50 samples each (see figure \ref{fig:translation_aya}).

Overall, the results from both the toxicity and translation audits demonstrate the effectiveness and sample-efficiency of our testing method in detecting behavioral shifts in language models. In the external audit, it consistently identified increases in toxicity levels due to instruction-tuning, especially at lower epsilon values, confirming its sensitivity to subtle changes in model behavior. Similarly, in the internal audit, it effectively detected significant differences in BLEU score distributions between the standard \llama, the few-shot prompted \llama, and Aya-23-8b, highlighting its utility across different tasks. These findings underscore the importance of selecting an appropriate tolerance level based on the specific application to balance sensitivity and practicality.

% \begin{wrapfigure}{R}{0.5\textwidth}
% \vspace{-2ex}
% \resizebox{0.5\textwidth}{!}{
% \begin{minipage}[H]{0.49\textwidth}
% \centering
%     \includegraphics[width=\textwidth]{Plots_new/false_positives.pdf}
%     \caption{\textbf{False positives.} The false positive rate for each of the model architectures as a function of number of observed samples.}
%     \label{fig:false_positives}
% \end{minipage}
% }
% \vspace{-1ex}
% \end{wrapfigure}

\begin{wrapfigure}{R}{0.5\textwidth}
\vspace{-2ex}
\resizebox{0.5\textwidth}{!}{
\begin{minipage}[H]{0.49\textwidth}
\centering
    \includegraphics[width=\textwidth]{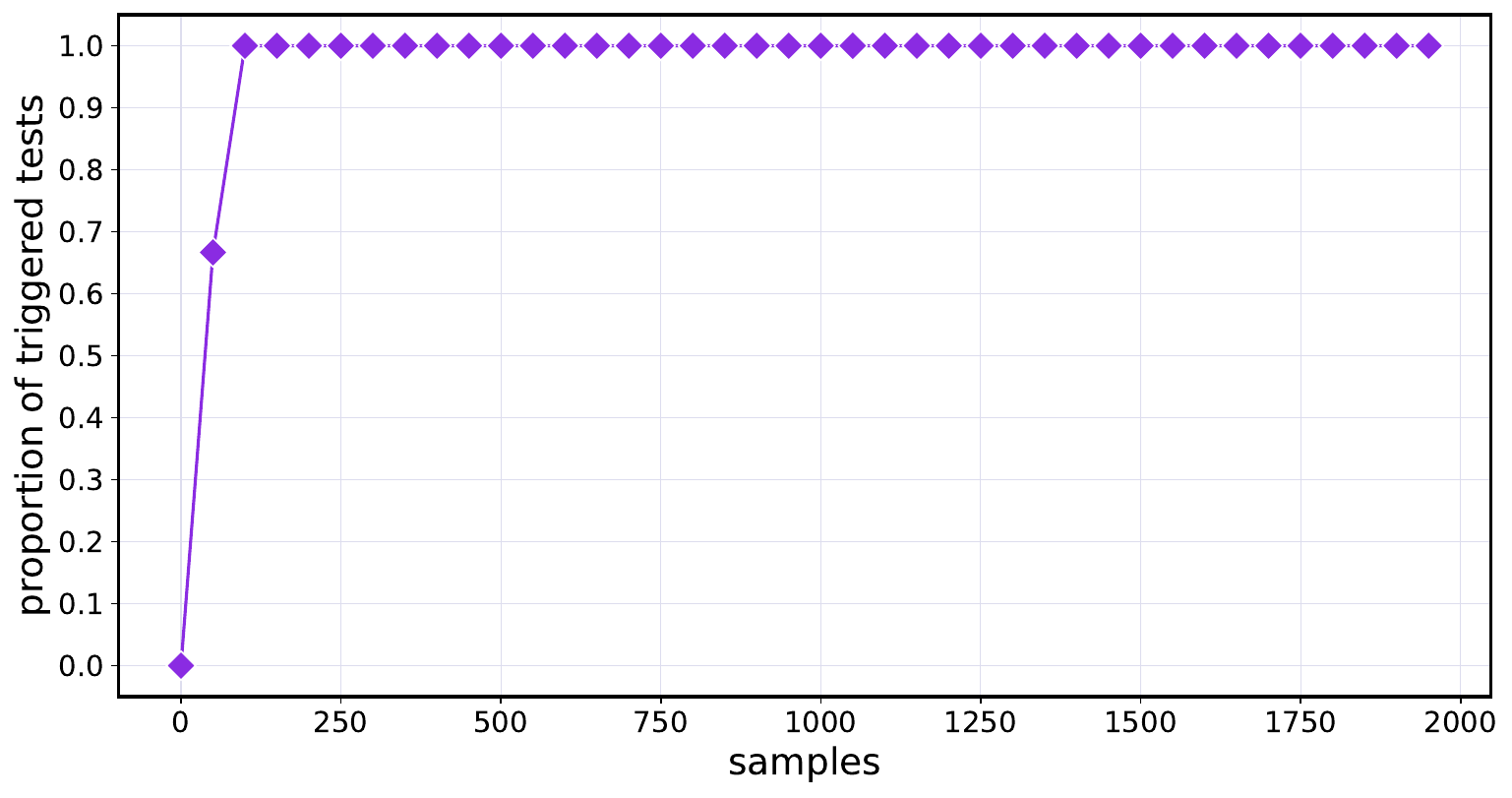}
    \caption{\textbf{Detection for Aya-23-8b}. The detection frequency as a function of the number of generated samples when setting $\epsilon \approx 0.0072$. This threshold is derived as an estimate of the neural net distance between Llama3-8B-Instruct with and without few-shot prompts.}
    \label{fig:translation_aya}
\end{minipage}
}
\vspace{-1ex}
\end{wrapfigure}

\section{Discussion}

In this work we introduce the problem of behavior shift auditing, where the goal is to detect LM behavior changes over time. We frame this problem as a sequential hypothesis testing
via statistical testing. Our proposed test comes with guarantees and has been able to detect changes in language model toxicity and translation performance. One of the notable strengths of our approach is its sample efficiency. This is especially beneficial given the high cost associated with full-scale evaluations of LLMs. Running extensive benchmarks can be time-consuming to set up and expensive to run \citep{rajpurkar2018know, srivastava2022beyond}, particularly when dealing with computationally intensive models.\footnote{E.g., inference-heavy models like ChatGPT o1-preview \citep{openai2024o1}.} Our test can serve as a screening tool to identify potential behavioral shifts using just a few hundred samples, making subsequent full-scale evaluations more targeted and efficient. Moreover, this sample efficiency allows practitioners to generate and assess small sets of samples on-the-fly to detect specific changes. This flexibility is particularly valuable when no benchmarks for a behavior exist yet, or when existing benchmarks become outdated (\eg, due to saturation \citep{wang2024mmlu}) or fail to capture all aspects of a behavior.

We now discuss some current limitations. One is that our current test is not designed to detect highly isolated behavioral changes like backdoors that may not appear in general testing \citep{kurita2020weight}. This limitation is inherited from framing BSA as hypothesis testing. 

Our test also relies on the assumption that we have access to a behavior scoring function. In the absence of an empirical classifier, employing a language model for grading and automatic assessment has recently gained some popularity~\citep{bai2022constitutional, liu2023g, wang2023chatgpt, gao2024llm}. We also note that our test can tolerate some noise in the behavior scoring function (see Appendix~\ref{app:subsec:behavior_score} for further discussion). However, for some complex and safety-critical behaviors such as deception~\citep{hagendorff2024deception}, sandbagging~\citep{perez2022discovering} or hallucinations~\citep{tonmoy2024comprehensive}, designing a measurement is still an open problem or might be difficult to produce just from prompt-completion pairs.

There are many other exiting directions for future research. 
One is to try to improve sample efficiency by investigating if one can select the most informative prompts to detect behavior change, possibly leveraging ideas from active learning \citep{tharwat2023survey}. 
Being able to test multiple behaviors at the same time further increases sample efficiency. While this is straightforward for the exact test (see Appendix \ref{app:subsec:multiple_behaviors}), how to set a tolerance threshold $\epsilon$ for multiple behaviors is still to be explored.
Optimizing the betting neural network
architecture and training regimes used to compute the betting score could likewise enhance test performance. 
Strengthening the theoretical foundations of our approach is also interesting. Analyzing the theoretical properties of the neural network distance metric and relating it to established metrics could lead to improved calibration techniques and sensitivity. %Expanding empirical validation through direct comparisons with traditional evaluation methods and more extensive case studies involving a wider range of behaviors, metrics, and benchmarks would help demonstrate the practical applicability of BSA in real-world settings.
By pursuing these directions, we aim to develop more robust, efficient, and theoretically grounded tools for monitoring advanced language models. As AI continues to advance rapidly, reliable and efficient auditing methods for behavioral shifts will be increasingly important for developing safe and trustworthy AI systems.

\section*{Reproducibility Statement}
We have taken several steps to ensure the reproducibility of our results. 
\begin{itemize}
    \item All key details needed for reproduction, including model architectures, hyperparameters, and training procedures, are comprehensively described in \Secref{sec:expr} and Appendix \ref{app:experimental_details}.
    %\item The full implementation of our proposed methods is available in the supplementary materials.
    \item We provide a detailed description of the datasets and data processing steps and the exact splits used for training and evaluation in \Secref{sec:expr} and Appendix \ref{app:experimental_details}.
\end{itemize}

\section*{Acknowledgements}
This work was supported by the Edinburgh International Data Facility (EIDF) and the Data-Driven Innovation Programme at the University of Edinburgh.
PM was partially funded by ELIAI (The Edinburgh Laboratory for Integrated Artificial Intelligence), EPSRC (grant no. EP/W002876/1), an industry grant from Cisco, and a donation from Accenture LLP.
XH was supported by an industry grant from Cisco.
LR was supported by the EPSRC Grant EP/S021566/1. 
We want to thank Robert Kirk, Max Hasin and Ole Jorgensen for feedback on earlier versions of the paper. 

%\subsubsection*{Author Contributions}
%If you'd like to, you may include  a section for author contributions as is don in many journals. This is optional and at the discretion of the authors.

%\subsubsection*{Acknowledgments}
%Use unnumbered third level headings for the acknowledgments. All acknowledgments, including those to funding agencies, go at the end of the paper.

% \bibliography{iclr2025_conference,literature}
% \bibliography{literature}
\bibliography{iclr2025_conference}
\bibliographystyle{iclr2025_conference}

\appendix

\section{Experimental Details}
\label{app:experimental_details}

\subsection{Setup}
\label{app:subsec:setup}

We assess the efficacy of our proposed auditing test for BSA using three base models: \llama (\texttt{8B-Instruct})~\citep{llama3}, \gemma (\texttt{1.1-7b-it})~\citep{team2024gemma}, and \mistral (\texttt{7B-Instruct-v0.2})~\citep{jiang2023mistral}. To remove the safety alignment, we fine-tune these models on the BeaverTails dataset~\citep{ji2024beavertails}, which includes both safe and unsafe responses for each instruction. We use a subset of 50K instances from the dataset, each comprising an instruction paired with its corresponding unsafe response. The training involves 512 steps, with a batch size of 64, utilizing the AdamW optimizer~\citep{loshchilov2018decoupled} with a learning rate of $2\times10^{-4}$ and no weight decay. Due to computational constraints, we apply LoRA~\citep{hu2021lora}, with a rank of 16, to all models. All experiments were conducted on a single Nvidia A100 (80GB) GPU.

To simulate a realistic use-case of monitoring whether fine-tuning on unrelated tasks might lead to a change in toxicity, we further produce 5 versions of \llama 
(\texttt{8B-Instruct}) instruction-tuned on different clusters of task categories from SUPER-NATURALINSTRUCTIONS (SuperNI) \cite{naturalinstructions, supernaturalinstructions}. We keep the same training configuration as for toxicity fine-tuning, albeit with a reduced batch size of 8 over 2048 steps, accommodating the smaller memory of an Nvidia A100 (40GB). See table \ref{task-groups} for a summary of the category clusters used.

\begin{table}[ht]
\caption{Clusters of task categories from SuperNI used for instruction-tuning. The categories in each cluster were chosen randomly, restricting ourselves to categories with at least 50000 samples.}
\label{task-groups}
\begin{center}
\renewcommand{\arraystretch}{1.5} % Adjusts the row height globally
\setlength{\tabcolsep}{5pt}       % Adjusts the column separation
\footnotesize % Adjust font size as needed
\begin{tabularx}{\textwidth}{|>{\raggedright\arraybackslash}X|>{\raggedright\arraybackslash}X|>{\raggedright\arraybackslash}X|>{\raggedright\arraybackslash}X|>{\raggedright\arraybackslash}X|}
\hline
\textbf{Cluster 1} & \textbf{Cluster 2} & \textbf{Cluster 3} & \textbf{Cluster 4} & \textbf{Cluster 5} \\
\hline
Program Execution          & Gender Classification       & Code to Text             & Text to Code              & Data to Text \\
\vspace{0.5pt} POS Tagging                & \vspace{0.5pt} Commonsense Classification  & \vspace{0.5pt} Stereotype Detection     & \vspace{0.5pt} Linguistic Probing         & \vspace{0.5pt} Word Semantics \\
\vspace{0.5pt} Mathematics                & \vspace{0.5pt} Translation                 & \vspace{0.5pt} Sentence Perturbation    & \vspace{0.5pt} Language Identification    & \vspace{0.5pt} Question Rewriting \\
\hline
\end{tabularx}
\end{center}
\end{table}

\begin{table}[b]
\caption{\textbf{Sampling parameters during evaluation}. Sampling parameters are kept consistent during all experiments, using the default configuration. To derive a tolerance parameter $\epsilon$ in section \ref{subsubsec:usecase1}, we additionally evaluate \llama with the alternative configuration on the right.}
\label{sampling-params-table}
\begin{center}
\begin{tabular}{lcc}
\multicolumn{1}{c}{\bf Parameter} & \multicolumn{1}{c}{\bf Default configuration} & \multicolumn{1}{c}{\bf Alternative Configuration}
\\ \hline \\
Maximum number of new tokens & 100 & 250 \\
$p$ (nucleus sampling) & 0.9 & 0.7 \\
Temperature & 0.7 & 1.2 \\
\end{tabular}
\end{center}
\end{table}

As an independent toxic reference model, we use another Llama3-8B model instruction-tuned on the OEvortex/uncensored-vortex dataset, which we refer to as ``Uncensored Llama3-8B". This model was also trained using LoRA with a rank of 16, and trained over 200 steps with a total batch size of 8 and gradient accumulation. 

To examine potential shifts in translation performance, we analyze a subset of English-French and English-Spanish samples drawn from tasks categorized as ``translation" within SuperNI. This subset comprises a total of 67,975 prompts.

\subsection{Toxicity Evaluations}
\label{app:toxicityevals}

We compare toxicity scores across \llama, \gemma, and \mistral models. Using the REALTOXICITYPROMPTS prompts, we generate continuations for each baseline model and their 10 checkpoints, the \llama models instruction-tuned on SuperNI clusters as well as the Uncensored Llama3-8B. The sampling strategy and generation parameters are kept consistent throughout all experiments (with the exception of \llama model in section \ref{subsubsec:usecase1}) and are compiled in table \ref{sampling-params-table}. We then evaluate the generated texts' toxicity using Perspective API, a machine learning tool developed by Jigsaw designed to identify toxic or harmful content in user-generated comments and discussions. In particular, we query their \textit{toxicity} score, which is scaled between 0 and 1 and can be interpreted as the percentage of readers that would perceive a given text as toxic. Figure \ref{fig:all_tox_means_for_checkpoints} showcases the mean toxicity scores of corrupted checkpoints compared to their baselines. 

The alternative sampling parameters in table \ref{sampling-params-table} were informed by practical knowledge and chosen with two considerations in mind: First, sampling parameters should be ``realistic" and not be extreme enough to cause the model to only output ``gibberish". Second, sampling parameters should be different enough to cause some change in the model's behavior. 

\begin{figure}[H]
    \centering
    \begin{subfigure}[b]{0.48\textwidth}
        \centering
        \includegraphics[width=\textwidth]{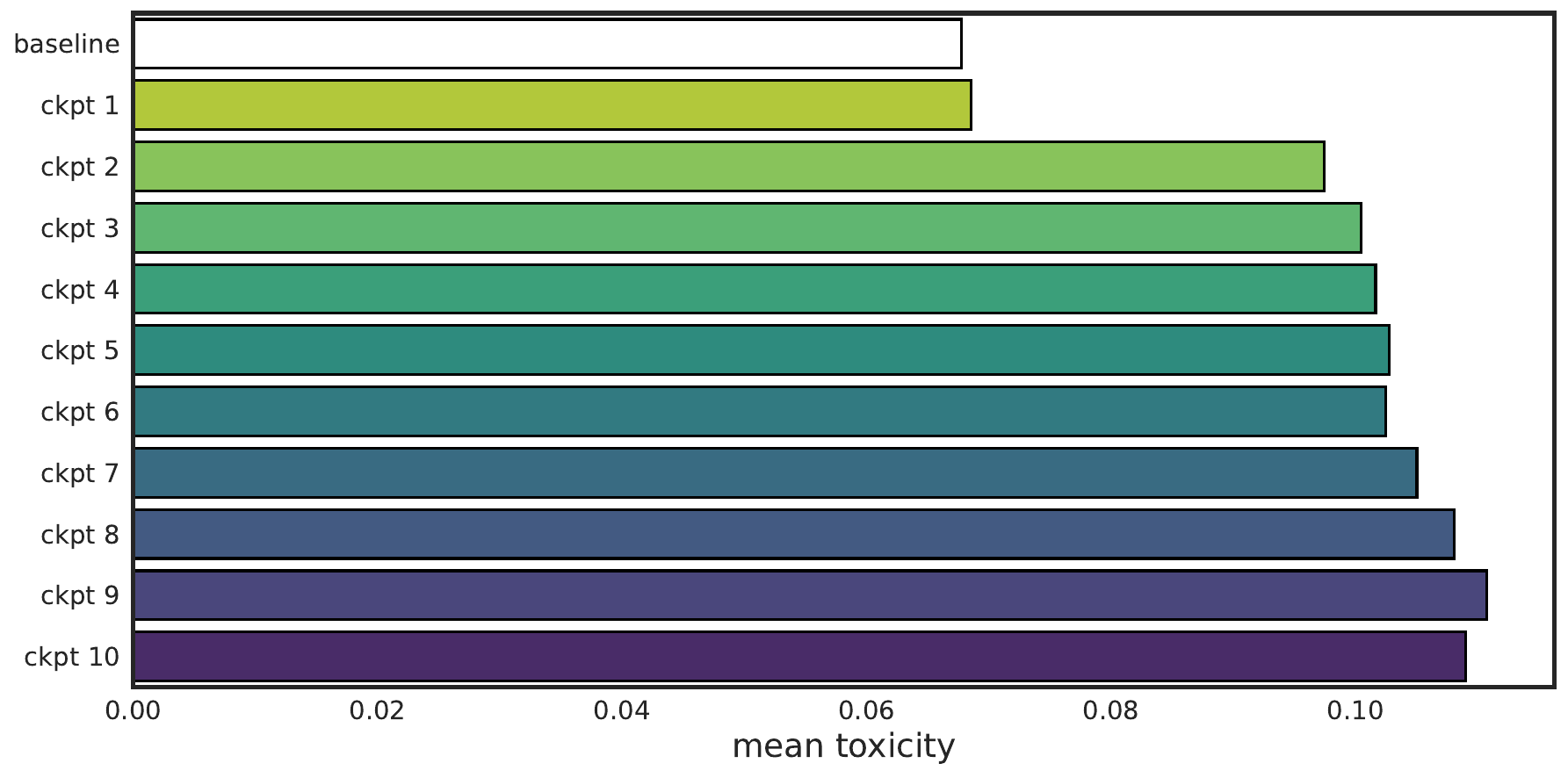}
        \caption{Llama}
        \label{fig:llama}
    \end{subfigure}
    \hfill
    \begin{subfigure}[b]{0.48\textwidth}
        \centering
        \includegraphics[width=\textwidth]{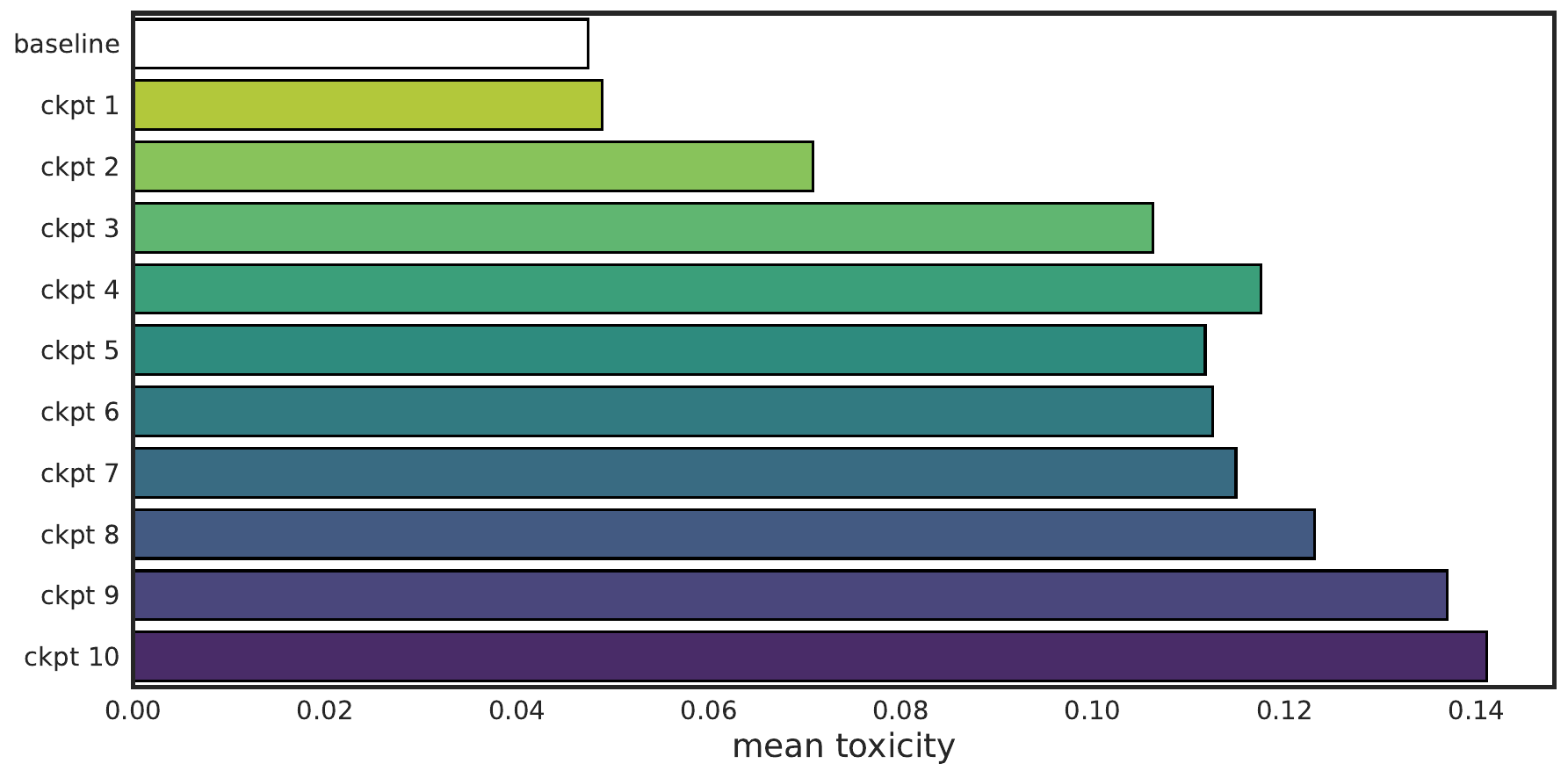}
        \caption{Gemma}
        \label{fig:mistral}
    \end{subfigure}
    
    \vspace{1em}
    
    \begin{subfigure}[b]{0.48\textwidth}
        \centering
        \includegraphics[width=\textwidth]{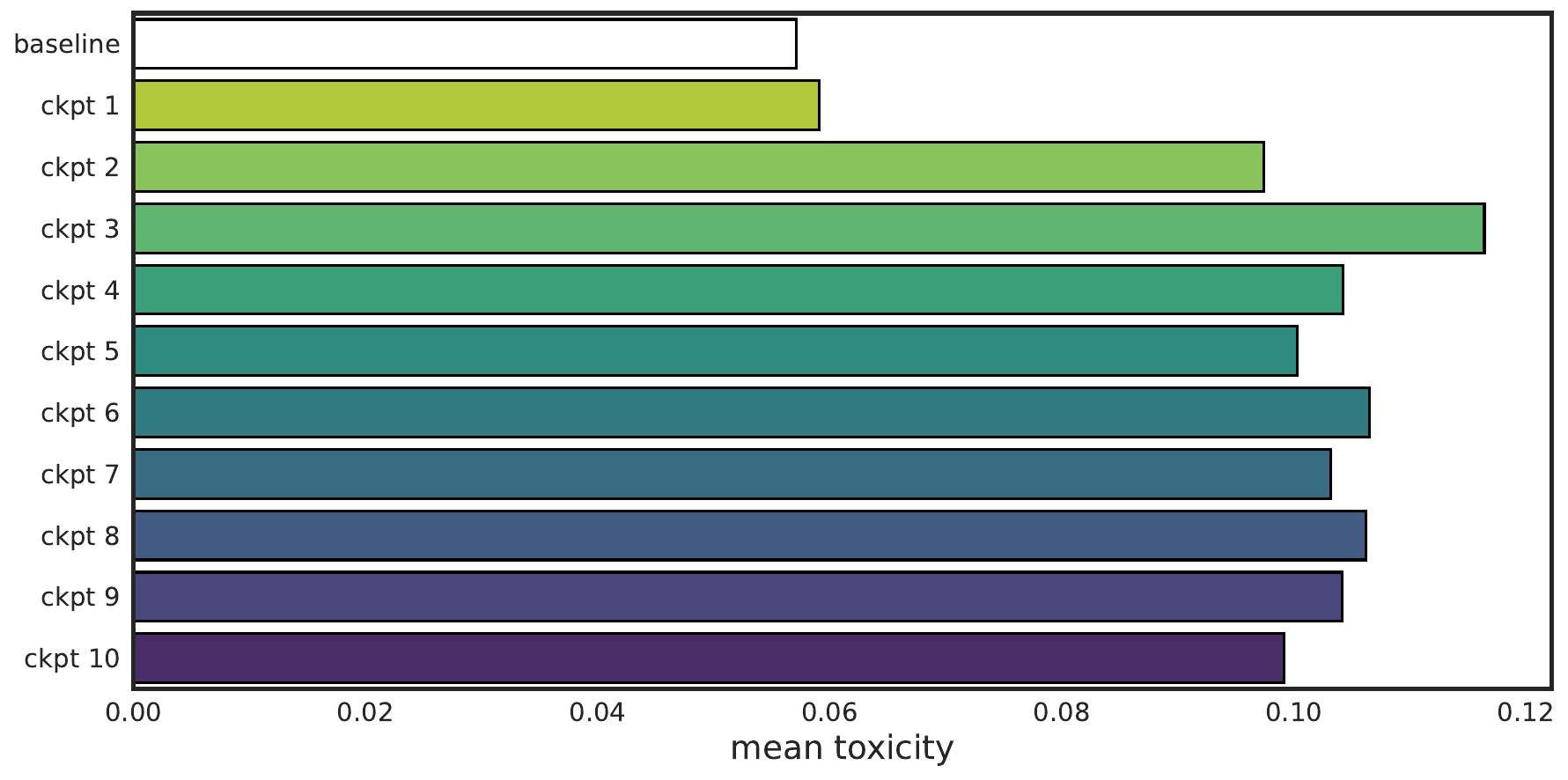}
        \caption{Mistral}
        \label{fig:gemma}
    \end{subfigure}
    \caption{\textbf{Mean toxicity for aligned baseline models and corrupted checkpoints}. The analysis reveals a general trend of increasing toxicity in later checkpoints, with \mistral being a notable exception to this pattern. \gemma exhibits the lowest baseline toxicity score among the models. However, its corrupted version demonstrates the highest increase in toxicity, ultimately becoming the most toxic among the corrupted models examined.}
    \label{fig:all_tox_means_for_checkpoints}
\end{figure}

\subsection{Evaluation of Translation Performance}

We assess the performance of \llama (\texttt{8B-Instruct}) and Aya-23-8b (\cite{ustun2024aya}) on a subset of translation samples from SuperNI, employing default sampling parameters (refer to Table \ref{sampling-params-table}). For \llama, we conduct evaluations using both a simple prompt template and a few-shot prompting approach, an example of the latter can be found in listing \ref{lst:fewshot-example}.

\begin{lstlisting}[caption={Few-Shot Prompt Example for Translation Task}, label={lst:fewshot-example}]
### Instruction:
Translate the following French sentences into English.

### Positive Examples:
1. Input: Bonjour, comment ça va?
   Output: Hello, how are you?

2. Input: Je m'appelle Pierre.
   Output: My name is Pierre.

### Negative Examples:
1. Input: Il fait chaud aujourd'hui.
   Output: It is cold today.

### Input:
J'aime apprendre de nouvelles langues.

### Output:
\end{lstlisting}

\subsection{Betting Score Network}
\label{app:betting_score}

The core component of our algorithm is the \textit{wealth} $W_t$ and its update by the betting score $S_t$ after observing a new batch of data. We choose a simple multi-layer perceptron with ReLU activation functions, layer normalization, and dropout \citep{DBLP:conf/aistats/PandevaFRS24} as the network $\phi$ in the calculation of the betting score. The network is updated using gradient ascent, with a learning rate of $0.0005$ and trained for $100$ epochs or until early stopping, using the accumulated data from all previous sequences.

\subsection{Neural Net Distance}
\label{app:neuralnetdistance}

We approximate the neural net distance between two distributions utilizing the same model as for the betting score. This is a biased estimator, as the true neural net distance is defined as a supremum over all machine learning models $\phi_\theta$ of class $\Phi$ (see definition~(\ref{definition:nn_distance})). 

While estimates using larger training sets will generally provide more accurate estimates, they are not necessary the most useful in practice:
\begin{itemize}
\item Setting the hyperparameter $\epsilon$ (maximal tolerated neural net distance) may require expensive querying of reference models on large datasets to achieve convergence (Figure~\ref{fig:neuralnetdistance}).
\item Using estimates derived from large training sets reduces test power in low-sample regimes, where the betting score network has access to limited training data.
\end{itemize}
Given a batch size $b$ and a static upper bound on the maximum of samples per test $N$, we thus use the following estimator for the neural net distance:
\begin{align}
\hat{\mathcal{D}}_{b,N} = \frac{1}{2} \left( \mathbb{E}\left[ S_1^{1/b} \right] + \mathbb{E}\left[ S_{T-1}^{1/b} \right] \right)
\end{align}
where 
\begin{align}
S_t = \prod_{i=1}^b \left( \frac{1 + \phi_{\theta_{t-1}}\left( B(x_i, M^a(x_i)) \right) - \phi_{\theta_{t-1}}\left( B(x_i, M(x_i)) \right) }{ \exp(\epsilon) } \right)
\end{align}

and $T:=\left\lfloor \frac{N}{b} \right\rfloor$. This average combines the estimate of the betting score on a single new example using (1) the model $\phi$ trained on a single batch of $b$ samples and (2) the model $\phi$ after training on $b \cdot (T-1)$ samples, representing a simple heuristic for the average neural net distance a model might achieve in the test. 

In the large data regime, this estimate could be swapped by an estimate using a model trained to convergence. Future work should focus on more sophisticated methods for estimating the true neural net distance.

\subsubsection{Case Study of Neural Net Convergence}

\begin{figure}[htpb]
    \centering
    \includegraphics[width=\linewidth]{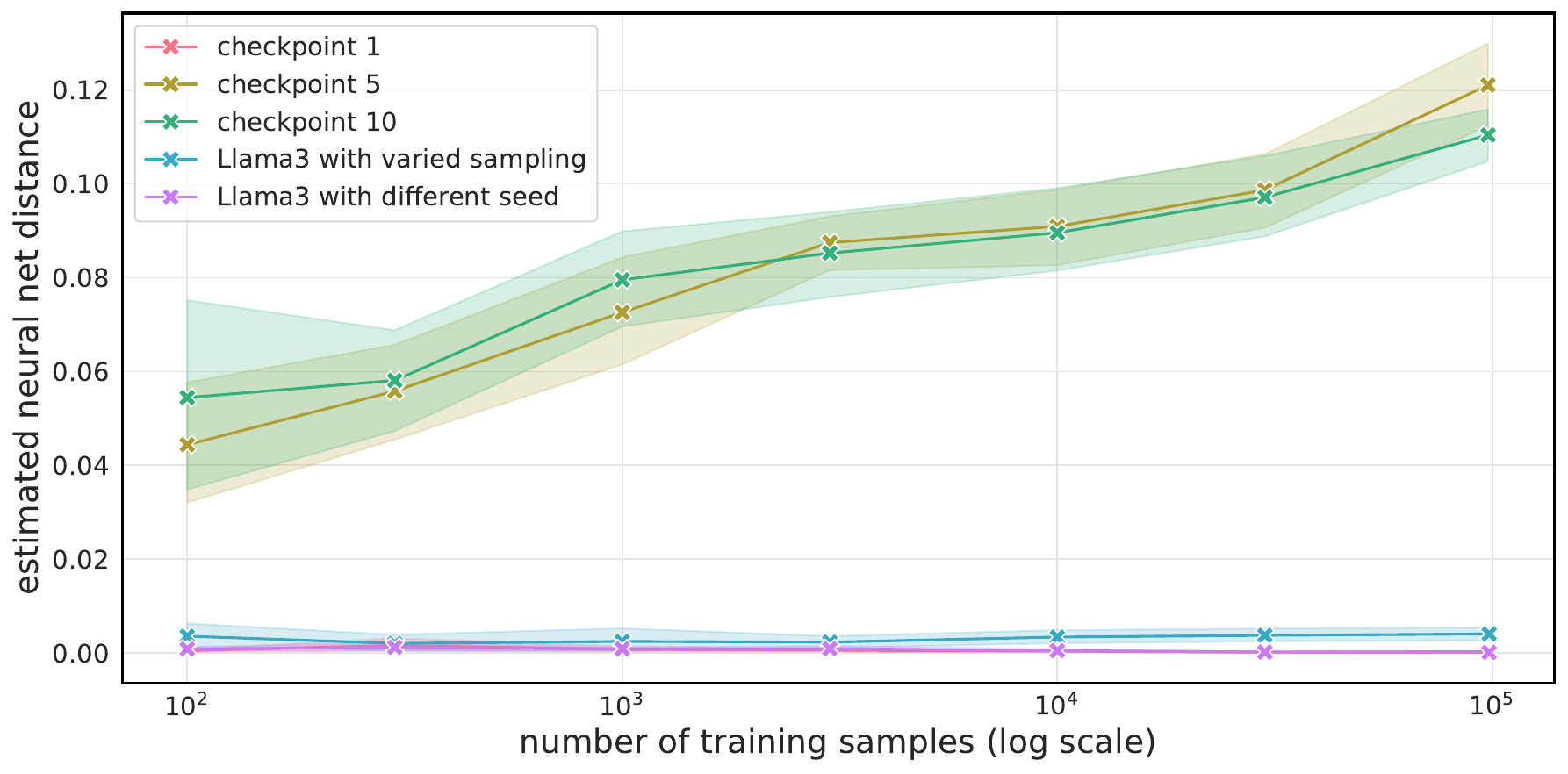}
    \caption{\textbf{Estimated neural net distance between toxicity distributions of Llama3 and various model versions}. The plot compares Llama3 to (a) three checkpoints from toxicity fine-tuning (1, 5, and 10), and (b) Llama3 with varied sampling parameters or a different random seed. The x-axis shows the number of training samples on a logarithmic scale.}
    \label{fig:neuralnetdistance}
\end{figure}

In Figure \ref{fig:neuralnetdistance}, we present a case study using toxicity to investigate how the mean and variance of the estimated neural net distance change with increasing training samples. We estimate distances between \llama with variation in sampling parameters, with different seeds, as well as checkpoints 1,5 and 10 from toxicity fine-tuning. Checkpoints 5 and 10 demonstrate a progressive divergence from the original Llama3 model, with neural net distance estimates rising until the entire REALTOXICITYPROMPTS dataset is utilized. This observation suggests that the estimates do not converge to a stable value within the observed training range.

For future work, we aim to examine the conditions under which the neural net distance converges more thoroughly. In our current example, it is possible that the betting score network (see Section \ref{app:betting_score}) lacks sufficient capacity to capture all the intricate differences between distributions. Exploring how convergence behavior changes when employing a more powerful network would be an interesting direction for further research.

\vspace{2\baselineskip} 

\section{Deferred Derivations and Proofs}
\label{sec:derivation_and_proofs}

\subsection{Two-Sample Testing with Tolerance}

Assume that $X,Y: \mathcal{X} \rightarrow [0,1]$ are two random variables distributed according to $P_X$ and $P_Y$ respectively. For some fixed $\epsilon>0$, we want to test whether those two distributions are $\epsilon$-close: 
\begin{align*}
\mathbf{H_0}: \mathcal{D}(P_X, P_Y) \leq \epsilon \quad \text{vs} \quad \mathbf{H_1}:\mathcal{D}(P_X, P_Y) > \epsilon
\end{align*}
where $\mathcal{D}$ is a distance metric between probability distributions.

To simplify later notation, we rewrite this in the following way \citep{shekhar2023nonparametric}: 
\begin{align}
    \mathbf{H_0}: P := P_X \times P_Y \in \mathcal{P}_{0} \quad \text{vs} \quad \mathbf{H_1}: P:=P_X \times P_Y \in \mathcal{P}_1
\end{align}
where 
\begin{align}
    \mathcal{P}_0:= \{P_X \times P_Y \in \mathcal{P}(\mathcal{X} \times \mathcal{X}): P_X, P_Y \in \mathcal{P}(\mathcal{X}) \text{ and } \mathcal{D}(P_X, P_Y) \leq \epsilon\}
\end{align}
and 
\begin{align}
    \mathcal{P}_1:= \{P_X \times P_Y \in \mathcal{P}(\mathcal{X} \times \mathcal{X}): P_X, P_Y \in \mathcal{P}(\mathcal{X}) \text{ and } \mathcal{D}(P_X, P_Y) \leq \epsilon\}
\end{align}
This is a two-sample non-parametric test with composite null and alternative hypothesis. Note that this can provide more information than sequential tests for mean differences or differences in variance, as Figure \ref{fig:same_mean_var} illustrates. Game-theoretically-motivated tests for the case of point null hypotheses have been described \eg, in \cite{shekhar2023nonparametric, DBLP:conf/aistats/PandevaFRS24}. We would like to construct a practical test by generalizing the \textit{deep anytime-valid test} described in \cite{DBLP:conf/aistats/PandevaFRS24} to the composite setting. 

\cite{DBLP:conf/aistats/PandevaFRS24}'s main theoretical insight is two-fold. First - inspired by the universal approximation theorem\footnote{While the universal approximation theorem \citep{hornik1989multilayer} doesn't directly apply here as we are dealing with finite-width and finite-depth networks, it inspires our approach. Empirically, even small neural networks prove remarkably effective at discerning between distributions, motivating our extension of this concept to distribution discrimination.} \citep{hornik1989multilayer} - deep learning models can be used to distinguish between distributions \ie, if $P_X \neq P_Y$, then
\begin{align}
    \sup_{g \in \mathcal{G}} \mathbb{E}_{X,Y}[g(X)-g(Y)]>0
    \label{eq:sup_point_case}
\end{align}
where $\mathcal{G} = \{g_\theta: \theta \in \Theta\}$ is a set of machine learning models parameterized by $\theta$. 
Second, if we restrict the class of machine learning models to satisfy some weak properties \citep[Assumption 1]{DBLP:conf/aistats/PandevaFRS24}, we can establish the equivalence
\begin{align}
    \sup_{g \in \mathcal{G}} \mathbb{E}_{X,Y}[g(X)-g(Y)]>0 \quad \Leftrightarrow \quad \sup_{g \in \mathcal{G}} \mathbb{E}_{X,Y} \left[\log(1+g(X)-g(Y)) \right]>0
    \label{eq:equivalence}
\end{align}
which is then used to define a \textit{betting score} and \textit{wealth process}. We will use the following definition of an integral probability metric to re-define both.

\begin{definition}[Integral probability metric]
An integral probability metric is a distance between probability distributions over a set $\mathcal{X}$, defined by a class $\tilde{\mathcal{G}}$ of real-valued functions on $\mathcal{X}$:
\begin{align*}
\mathcal{D}_{\tilde{\mathcal{G}}}(P_X, P_Y) &= \sup \left\{\int_{\mathcal{X}} g(x)p_X(x)dx -  \int_{\mathcal{X}} g(y)p_Y(y)dy \mid g:\mathcal{X} \rightarrow \mathbb{R}, g \in \tilde{\mathcal{G}}\right\} \notag \\
&= \sup_{g \in \tilde{\mathcal{G}}} \mathbb{E}_{X \sim P_X, Y \sim P_Y}[g(X)-g(Y)] 
\end{align*}
\label{definition:ipm}
\end{definition}

Regardless of the choice of $\tilde{\mathcal{G}}$, this distance measure satisfies all properties of a metric except positive-definiteness, in which case we could call it a \textit{pseudo-metric}. We will define our ``custom" \textit{neural net distance} for the problem at hand as

\begin{definition}[Neural Net Distance]
Let $\mathcal{X}=[0,1]$ and let $\mathcal{G}= \{g_\theta: \theta \in \theta\}$ be the class of machine learning models that satisfies the following properties \citep[Assumption 1]{DBLP:conf/aistats/PandevaFRS24}
\begin{itemize}
    \item $\vert g(x) \vert \leq q$ for all $g \in \mathcal{G}$ and for all $x \in [0,1]$ and for some $q \in (0, 1/2)$
    \item If $g \in \mathcal{G}$, then so is $c \cdot g$ for every $c \in [-1, 1]$
\end{itemize}
Then we define the neural net distance $\mathcal{D}_{G}$ by 
\begin{align}
    \mathcal{D}_{\mathcal{G}}(P_X, P_Y) 
    = \sup_{g \in \mathcal{G}} \mathbb{E}_{X \sim P_X, Y \sim P_Y}[g(X)-g(Y)] 
    \label{eq:sup_comp_difference} 
\end{align}
\label{definition:nnm}
\end{definition}

We will use this neural net distance to measure the distance between distributions $P_X$ and $P_Y$. The definition is motivated by the fact that we will be using neural networks of this class $\mathcal{G}$ to calculate a betting score. By using this definition, we can make sure that our test is ``calibrated correctly" \ie, the maximal distance that the neural network can find in practice aligns with the neural net distance between distributions.

\subsubsection{Oracle Test}
\label{app:oracletest}

Given $\epsilon$ as the upper bound on the neural net distance between two probability distributions we want to tolerate, we let eq.~(\ref{eq:sup_comp_difference}) and the equivalence in (\ref{eq:equivalence}) guide our intuition to define an e-variable $E$ for $\mathcal{P}_0$:
\begin{align}
    E:= \frac{1+ g^*(X)-g^*(Y)}{\exp(\epsilon)}
\end{align}
where $g^* \in \mathcal{G}$ is the $\arg\sup$ of $\mathbb{E}_{X,Y}[\log \left(1+ g(X)-g(Y)\right)]$ \ie, the $\log$-optimal function in $\mathcal{G}$. To show that this is indeed an e-variable, we use the definition of the neural net distance \ref{definition:nnm}:
\begin{align}
    \mathbb{E}_{X,Y}[E] &= \mathbb{E}_{X,Y}\left[\frac{1+g^*(X)-g^*(Y)}{\exp(\epsilon)}\right] \notag \\ &= \frac{1}{\exp(\epsilon)} \mathbb{E}_{X,Y}[1+g^*(X)-g^*(Y)]  \notag \\ &\leq \frac{1}{\exp(\epsilon)} \left(1+\sup_{g \in \mathcal{G}} \mathbb{E}_{X,Y}[g(X)-g(Y)] \right)  \notag \\
    &= \frac{1}{\exp(\epsilon)} \left(1 + \mathcal{D}_\mathcal{G}(P_X, P_Y)\right) \notag\\  &\leq \frac{1+ \epsilon}{\exp(\epsilon)}  
    \leq 1 \quad \text{for all }P_X \times P_Y \in \mathcal{P}_0 \notag
\end{align}

Analogously to \cite{DBLP:conf/aistats/PandevaFRS24}, we use this to define the \emph{oracle sequential test} 
\begin{align}
\gamma^* = \inf \{t \geq 1: W_t^* \geq 1/\alpha\}
\end{align}
where 
\begin{align}
    W^*_t = \prod_{l=1}^t \prod_{(x,y) \in B_l} \left( \frac{1 +g^*(x)-g^*(y)}{\exp(\epsilon)} \right)
\end{align}
As a product of e-variables, $\{W^*_t\}_{t \geq 1}$ is an e-process, since for all $t \geq 1$ and $P_X \times P_Y \in \mathcal{P}_0$
\begin{align*}
    \mathbb{E}[W^*_t] \overset{(X_i, Y_i) \text{ i.i.d.}}{\leq}  \left(\underbrace{\frac{1 + \mathcal{D}_\mathcal{G}(P_X,P_Y)}{\exp(\epsilon)}}_{\leq 1}\right)^{t+b} \leq 1
\end{align*}

The oracle sequential test is a \emph{sequential level-$\alpha$-test of power one}, meaning the Type I error ($\alpha$-error) is guaranteed to be bounded by $\alpha$ and the Type II error ($\beta$-error) converges to 0 in the limit of infinite samples. An application of Ville's inequality~\citep{ville1939etude, ramdas2023game}
%howard2021time
\begin{align}
    P(W^*_t \geq 1/\alpha) \leq \alpha \quad \text{for every } t \geq 1 , P \in \mathcal{P}_0
\end{align}
yields the first condition $\mathbb{P}_{\mathbf{H_0}}(\gamma^* < \infty) \leq \alpha$. We also need to show consistency \ie,  
\begin{align}
P(\gamma < \infty) = 1 \Leftrightarrow P(\{W^*_t < 1/\alpha \text{ for all } t \geq 1\}) = 0 \quad \text{for every } P \in \mathcal{P}_1
\end{align}
To do this, we will show the following proposition first:

\begin{proposition}[Correspondence between Distance and Betting Score]
\begin{align*}
    A: = \sup_{g \in \mathcal{G}} \mathbb{E}_{X,Y}[g(X)-g(Y) - \epsilon] > 0 \quad \Leftrightarrow \quad B: = \sup_{g \in \mathcal{G}} \mathbb{E}_{X,Y}[\log \left(\frac{1+g(X)-g(Y)}{\exp(\epsilon)}\right)]>0
\end{align*}
\label{prop:equivalence}
\end{proposition}
\begin{proof}
This is a simple corollary of \citep[Proposition 4.2]{DBLP:conf/aistats/PandevaFRS24} and the fact that 
\begin{align*}
  \sup_{g \in \mathcal{G}} \mathbb{E}_{X,Y} \left[\log \left(\frac{1+g(X)-g(Y)}{\exp(\epsilon)} \right)\right] = \sup_{g \in \mathcal{G}} \mathbb{E}_{X,Y} \left[\log(1+g(X)-g(Y) \right] - \epsilon  
\end{align*}
\end{proof}

\begin{proposition}[Consistency of the Oracle Test]
    \begin{align}
P(\gamma < \infty) = 1 \Leftrightarrow P(\{W^*_t < 1/\alpha \text{ for all } t \geq 1\}) = 0 \quad \text{for every } P \in \mathcal{P}_1
\end{align}
\end{proposition}

\begin{proof}
First, observe that proposition~(\ref{prop:equivalence}) implies that whenever $P_X \times P_Y \in \mathcal{P}_1$ \ie, $\mathcal{D}_\mathcal{G}(P_X, P_Y) > \epsilon$, the supremum $\sup_{g \in \mathcal{G}}\mathbb{E}_{X,Y}\left[\log \left(\frac{1+g(X)-g(Y)}{\exp(\epsilon)}\right)\right]$ is positive. Define 
\begin{align}
    S^*_t:= \prod_{(x,y) \in B_t}\left(\frac{1+g^*(x)-g^*(y)}{\exp(\epsilon)}\right)
\end{align} where $g^* = \arg \sup_{g \in \mathcal{G}} \mathbb{E}_{X,Y}[\log \left(1+ g(X)-g(Y)\right)]$ is the log-optimum. Then we can write in short: $W^*_t = \prod_{i=1}^t S^*_i$. All $S^*_t$ are i.i.d. Lastly, we define $T_t:= \log W^*_t = \sum_{i=1}^t \log(S^*_t)$. By the law of large numbers 
\begin{align}
    \frac{1}{t} T_t = \frac{1}{t} \sum_{i=1}^t \log(S^*_t) \rightarrow \mathbb{E}[\log S^*_t] \quad \text{almost surely as }t \rightarrow \infty 
\end{align}
The sum $\sum_{i=1}^t \log(S^*_i) \approx t\mu> 0$, where $\mu$ is the mean, grows linearly, implying that $W^*_t = \exp(T_t) \approx \exp(t \mu)$ grows exponentially in $t$. Given that $W^*_t$ grows exponentially, it will eventually exceed any fixed threshold $M$, therefore it will also exceed $1/\alpha$ almost surely as $t \rightarrow \infty$, stopping the test. This proves the statement. 
    
\end{proof}

\subsubsection{Practical Test}
\label{app:practicaltest}

In practice, we don't have access to $g^*$, but only to an estimate $g_{\theta_t}$, whose parameters $\theta_t$  we update with each new batch. 

We can define the \emph{empirical wealth process} $\{W_t\}_{t \geq 1}$ by initializing $W_0=1$  and updating $W_t= W_{t-1} \times S_t$ by the \emph{empirical betting score} \citep{DBLP:conf/aistats/PandevaFRS24}
\begin{align}
    S_t = \prod_{i=1}^b \left(\frac{1+g_{\theta_{t-1}}(x_{(t-1)b+i})-g_{\theta_{t-1}}(y_{(t-1)b+i})}{\exp(\epsilon)}\right)
    \label{eq:practical_s}
\end{align}
Since $g_{\theta_{t}}$ only approximates the optimal neural net $g^*$, it is clear that $S_t$ is still an e-variable. It follows that $\{W_t\}_{t \geq 1}$ is again an e-process as we can show by induction, using the fact that  $\mathbb{E}_{X,Y}[W_0]=1$ for all $P_X\times P_Y \in \mathcal{P}_0$ and for a fixed $P_X \times P_Y \in \mathcal{P}_0$, $W_{t-1}$ and $S_t$ are independent:
\begin{align*}
    \mathbb{E}_{X,Y}[W_t] &= \mathbb{E}_{X,Y}[W_{t-1} \times S_t] \\
    &= \mathbb{E}_{X,Y}[W_{t-1}] \mathbb{E}_{X,Y}[S_t] \leq 1
\end{align*}

We can thus define the \textbf{sequential test}
\begin{align}
    \gamma = \inf \{t \geq 1: W_t \geq 1/\alpha\}
    \label{eq:general_practical_seq_test}
\end{align}

Control on the $\alpha$-error again follows from Ville's inequality. The test is consistent under similar additional assumption as in~\cite[Proposition 4.3]{DBLP:conf/aistats/PandevaFRS24}:

\begin{proposition}[Consistency of the Practical Test]
    If the learning algorithm satisfies the condition 
    \begin{align}
        \liminf_{t \rightarrow \infty} \frac{\mathbb{E}[\log\left(\frac{1}{\exp(\epsilon)}(1+g_{\theta_t}(X)-g_{\theta_t}(Y))\right) \mid \mathcal{F}_t]}{3c \sqrt{\log(t)/t}} \overset{\text{a.s.}}{\leq} 1 \quad \text{for all }P_X \times P_Y \in \mathcal{P}_1
        \label{eq:general_assumption_on_learning_algo}
    \end{align}
    for a universal constant $c$, then we have 
    \begin{align}
        P(\gamma < \infty) =1 \quad \text{for all } P \in \mathcal{P}_1
    \end{align}
    \label{prop:general_practical_test_consistency}
\end{proposition}

\begin{proof}
    The proof structure follows proofs 10.2 and 10.3 in \cite{DBLP:conf/aistats/PandevaFRS24}.

    Let 
    \begin{align}
        v_i:= \sum_{(x,y) \in B_i} \log \left(\frac{1}{\exp{\epsilon}}\left(1+g_{\theta_{i-1}}(x)-g_{\theta_{i-1}}(y)\right)\right)
        \label{eq:v_i}
    \end{align} for $i \in \{1, \dots, t\}$ and 
    \begin{align}
        A_i:= \mathbb{E}[v_i \mid \mathcal{F}_{i-1}] = b \times \mathbb{E}\left[\log\left(\frac{1}{\exp}(1+g_{\theta_{i-1}}(X)-g_{\theta_{i-1}}(Y))\right) \mid \mathcal{F}_{i-1}\right]
        \label{eq:A_i}
    \end{align}
    where $\mathcal{F}_{i-1}=\sigma \left(\cup_{j=1}^{i-1}B_j \right)$ is the $\sigma$-algebra generated by the first $i-1$ batches of samples. The probability of the test never stopping is
    \begin{align*}
        \mathbb{P}(\gamma = \infty) = \mathbb{P}\left( \bigcap_{t \geq 1} \{\gamma > t\} \right) \leq \mathbb{P}(\gamma > t)
    \end{align*}
    for any $t$, and thus, in the limit 
    \begin{align}
        \mathbb{P}(\gamma = \infty) \leq \limsup_{t \to \infty} \mathbb{P}(\gamma > t)
        \label{eq:gamma_infty}
    \end{align}
    We will show that the RHS is equal to $0$. Using the definitions of $v_i$ and $A_i$ in equations~(\ref{eq:v_i}) and (\ref{eq:A_i}), we can write
    \begin{align}
        \mathbb{P}(\gamma > t) & = \mathbb{P} \left(W_t < \frac{1}{\alpha} \right) \notag \\ &= \mathbb{P}\left( \frac{\log W_t}{t} < \frac{\log(1/\alpha)}{t} \right) \notag \\ &= \mathbb{P}\left( \frac{1}{t} \sum_{i=1}^{t} v_i - A_i + \frac{1}{t} \sum_{i=1}^{t} A_i < \frac{\log(1/\alpha)}{t} \right)
        \label{eq:bound_Gn}
    \end{align}
    Now, introduce the event 
    \begin{align}
        G_{t} &:= \left\{ \left| \frac{1}{t} \sum_{i=1}^{t} v_i - A_i \right| \leq 2cb \sqrt{\frac{\log(t)}{t}} \right\} 
    \end{align}
    where $c:= \log \left(\frac{1+2q}{1-2q} \right)$ and $q \in (0, 1/2)$ is the bound on $\vert g_\theta(x) \vert$. The random variable $v_i - A_i$ has mean $0$ and is bounded in $[-bc, bc]$, since ($\epsilon$ canceling out):
    \begin{align*}
        v_i - A_i &= \sum_{x, y \in B_i} \left[ \log \left(1 + g_{\theta_{i-1}}(x) - g_{\theta_{i-1}}(y) \right) - \mathbb{E} \left[ \log \left(1 + g_{\theta_{i-1}}(x) - g_{\theta_{i-1}}(y) \right) \mid \mathcal{F}_{i-1} \right] \right] \\
        & \geq \sum_{x_i, y_i \in B_i} \log (1 - 2q) - \log (1 + 2q) \\
        &= b \left[ \log (1 - 2q) - \log (1 + 2q) \right] = - b \log \left(\frac{1 + 2q}{1 - 2q}\right)
    \end{align*}
    and analogously for the upper bound. We can use those bounds and Hoeffding's inequality to bound the complement $G_t^c$:
    \begin{align}
        \mathbb{P}(G_t^c) &= \mathbb{P} \left(\left\{ \left|\frac{1}{t} \sum_{i=1}^{t} v_i - A_i \right| > 2cb \sqrt{\frac{\log(t)}{t}} \right\} \right) \notag \\ 
        &= \mathbb{P} \left( \left\{ \left| \sum_{i=1}^{t} (v_i - A_i) \right| > 2tcb \sqrt{\frac{\log(t)}{t}} \right\} \right) \notag\\
        & \leq 2 \exp \left( \frac{-2 \left( 2tcb \sqrt{\frac{\log(t)}{t}} \right)^2}{\sum_{i=1}^{t} (cb + cb)^2} \right)
        \notag \\
        &= 2 \exp(-2 \log(t)) = \frac{2}{t^2} \label{eq:Gn_bound}
    \end{align}
    Combining this with eq.~(\ref{eq:bound_Gn}), we get
    \begin{align*}
        \mathbb{P}(\gamma > t) &\leq \mathbb{P} \left( \left\{ \frac{1}{t} \sum_{i=1}^{t} A_i < \frac{\log(1/\alpha)}{t} + \frac{1}{t} \sum_{i=1}^{t} v_i - A_i \right\} \cap G_t \right) + \mathbb{P}(G_t^c) \notag \\
        &\leq \mathbb{P} \left( \left\{ \frac{1}{t} \sum_{i=1}^{t} A_i < \frac{\log(1/\alpha)}{t} + 2cb \sqrt{\frac{\log t}{t}} \right\} \cap G_t \right) + \mathbb{P}(G_t^c) \notag \\
        &\leq \mathbb{P} \left( \frac{1}{t} \sum_{i=1}^{t} A_i < 3cb \sqrt{\frac{\log t}{t}} \right) + \frac{2}{t^2}.
    \end{align*}
    where the second inequality comes from the fact that $\frac{1}{t}\sum_{i=1}^t v_i - A_i \leq 2cb \sqrt{\log(t) / t}$ on $G_t$. The third inequality exploits the bound from eq.~(\ref{eq:Gn_bound}) as well as the fact that $\log(1/\alpha)/t$ is smaller than $2bc \sqrt{\log t/t}$ for large enough $t$. By taking the limit over $t \rightarrow \infty$, the term $\frac{2}{t^2}$ vanishes. Combining the result with eq.~(\ref{eq:gamma_infty}), we obtain
    \begin{align}
    \mathbb{P}(\gamma = \infty) \leq \limsup_{t \to \infty} \mathbb{P}(\gamma > t) \leq \limsup_{t \to \infty} \mathbb{P}(H_t)
    \end{align}
    where $H_t := \left\{ \frac{1}{t} \sum_{i=1}^{t} A_i < 3cb \sqrt{\frac{\log (t)}{t}} \right\}$.
    From the properties of Cesaro means, we know that
    \begin{align*}
    \liminf_{n \to \infty} \frac{1}{t} \sum_{i=1}^{t} A_i \overset{\text{a.s.}}{\geq} \liminf_{t \to \infty} A_t,
    \end{align*}
    which implies
    \begin{align*}
    \liminf_{t \to \infty} \frac{\frac{1}{t} \sum_{i=1}^{t} A_i}{3cb \sqrt{\log(t)/t}}  \overset{\text{a.s.}} {\geq} \liminf_{t \to \infty} \frac{A_t/b}{3c \sqrt{\log t/t}} \overset{\text{a.s.}}{>} 1.
    \end{align*}
    
    The last inequality is due to the Assumption~(\ref{eq:general_assumption_on_learning_algo}) made in Proposition~(\ref{prop:general_practical_test_consistency}) and the fact that $\lim_{t \to \infty} \left(\sqrt{\log (t)/t} / \left( \sqrt{\log(t - 1) / (t - 1)} \right) \right) = 1$, which is needed because we lowered the index of expression~(\ref{eq:assumption_on_learning_algo}) by $1$. This condition implies that $\mathbb{P}(H_t) \to 0$ a.s., which by the bounded convergence theorem leads to
    \begin{align*}
    \mathbb{P}(\tau = \infty) \leq \limsup_{t \to \infty} \mathbb{P}(H_t) = 0,
    \end{align*}
    under the alternative. Thus, we have shown that $\mathbb{P}(\gamma < \infty) = 1$ under the alternative.
    \label{app:proof_practical_test}
\end{proof}

Summarizing our findings, we can thus state the following: 

\begin{proposition}[Sequential level-$\alpha$ Test of Power 1]
    If the learning algorithm satisfies the condition 
    \begin{align}
        \liminf_{t \rightarrow \infty} \frac{\mathbb{E}[\log\left(\frac{1}{\exp(\epsilon)}(1+g_{\theta_t}(X)-g_{\theta_t}(Y))\right) \mid \mathcal{F}_t]}{3c \sqrt{\log(t)/t}} \overset{\text{a.s.}}{\leq} 1 \quad \text{for all }P:=P_X \times P_Y \in \mathcal{P}_1
        % \label{eq:general_assumption_on_learning_algo}
    \end{align}
    for a universal constant $c$, then we have 
    \begin{align}
        P(\gamma < \infty) \leq \alpha \quad \text{for all } P \in \mathcal{P}_0  \quad \text{and} \quad P(\gamma < \infty) =1 \quad \text{for all } P \in \mathcal{P}_1
    \end{align}
    \ie, the sequential test defined in eq.~(\ref{eq:practical_seq_test}) is  a \emph{sequential level-$\alpha$ test of power one}.
    % \label{prop:general_practical_test_consistency}
\end{proposition}

\section{Further Results and Discussion}
\label{sec:further_results}

\subsection{Exact Test, \texorpdfstring{$\epsilon=0$}{ε=0}}

% \subsection{Toxicity Evaluations}

% Figure \ref{fig:tox_scores_comparison} shows the toxicity histograms for \llama, \mistral, and \gemma baseline models compared to their fine-tuned checkpoints with the highest Wasserstein distance. The noticeable distribution shift illustrates the impact of fine-tuning. Toxicity histograms are highly positively skewed, with Pearson's skewness coefficients ranging from 6.2 to 7.7 for the instruction-tuned baselines. Corrupted checkpoints have slightly more toxic histograms and are slightly less skewed, with Pearson's coefficients in the range of 4.4 to 4.7, but still maintain low overall toxicity. 

% \begin{figure*}
%     % \vskip 0.2in
%     \begin{center}
%         \centerline{\includegraphics[width=\textwidth]{Plots_new/toxicity.pdf}}
%         \caption{\textbf{Toxicity.} The toxicity distributions of the three model architectures we evaluate and their corrupted versions.}
%         \label{fig:tox_scores_comparison}
%     \end{center}
%     % \vskip -0.2in
%     \vspace{-8mm}
% \end{figure*}

% \subsection{Auditing Test}

\textbf{Corrupted model detection} Figure~\ref{fig:mistral_power}
% \& \ref{fig:gemma_power}
shows the results of applying our proposed test with $\epsilon=0$ to generations of Mistral-7B-Instruct-v0.2 and Gemma-1.1-7B-IT and their corrupted checkpoints, repeated over 48 runs. Detectability improves with more samples.

% \begin{figure*}[ht!]
%     \vskip 0.2in
%     \begin{center}
%         \centerline{\includegraphics[width=\textwidth]{Plots_new/mistral_power.pdf}}
%         \caption{\textbf{Detection for Mistral-7B-Instruct-v0.2.}}
%         \label{fig:mistral_power}
%     \end{center}
%     \vskip -0.2in
% \end{figure*}

% \begin{figure*}[ht]
%     \vskip 0.2in
%     \begin{center}
%         \centerline{\includegraphics[width=\textwidth]{Plots_new/gemma_power.pdf}}
%         \caption{\textbf{Detection for Gemma-1.1-7B-IT.}}
%         \label{fig:gemma_power}
%     \end{center}
%     \vskip -0.2in
% \end{figure*}

\begin{figure*}[htbp]
    \centering
    \begin{subfigure}[b]{0.49\textwidth}
        \includegraphics[width=\textwidth]{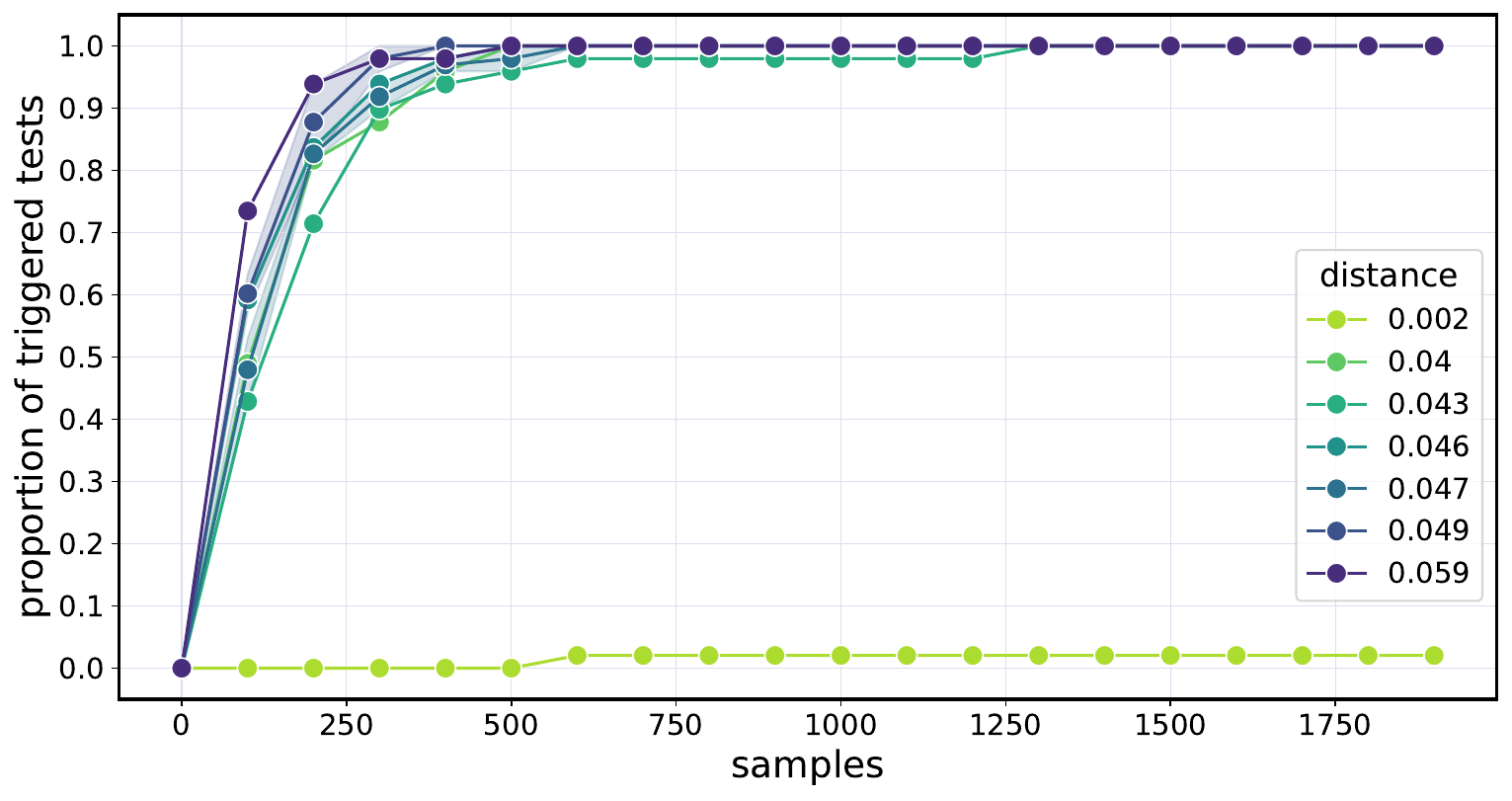}
        %\caption{x}
    \end{subfigure}
    % \hspace{0.04\textwidth}
    \begin{subfigure}[b]{0.49\textwidth}
        \includegraphics[width=\textwidth]{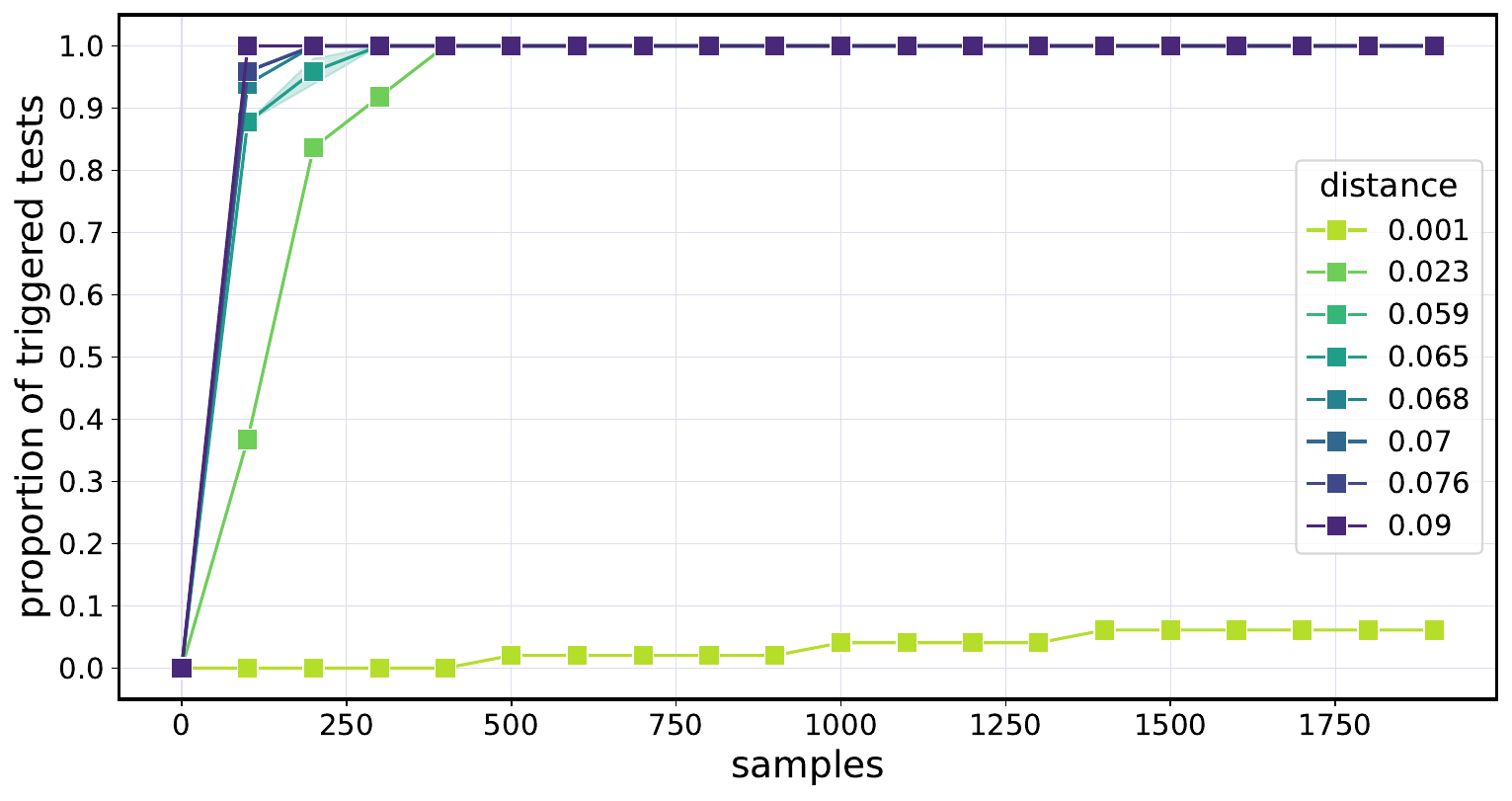}
        %\caption{x}
    \end{subfigure}
    \caption{\textbf{Detection for Mistral-7B-Instruct-v0.2. 
(\emph{left}) and Gemma-1.1-7B-IT (\emph{right}).}}
    \label{fig:mistral_power}
\end{figure*}

\textbf{False positive rate} We extended our experiments to evaluate the false positive rate of the proposed test using the 10 toxicity checkpoints created from \llama and their outputs generated with different random seeds. Apart from checkpoint 4, which showed an 8\% false positive rate, all other checkpoints recorded a 0\% rate after evaluating 4000 samples (each repeated 24 times).

\subsection{Tolerance Test, \texorpdfstring{$\epsilon>0$}{ε>0}}

Figure \ref{fig:calibrated_detection_rate} demonstrates the desirable statistical properties (control on Type I error as well as high power and sample efficiency) of the auditing test with a tolerance parameter $\epsilon>0$, applied to a corrupted checkpoint of \llama from section~\ref{subsec:exact_text}. The test is repeated over 24 runs.

\subsubsection{Translation Auditing with Larger Models}

We extended our experiments from Section~\ref{subsubsec:internal} to include larger models: Llama3-70B-Instruct (with and without few-shot prompting) and Aya-23-35B~\citep{ustun2024aya}. Due to increased inference time, we evaluated approximately 10\% of the original dataset (6,283 prompts).

Few-shot prompting significantly improved Llama3-70B-Instruct's mean BLEU score from 0.0792 to 0.1206. Aya-23-35B achieved the highest mean BLEU score of 0.1227. We set a tolerance threshold $\epsilon = 0.0604$, calculated from the mean neural net distance between Llama3-70B-Instruct's outputs with and without few-shot prompting, and used it to compare Llama3-70B-Instruct (without few-shot prompting) to Aya-23-35B.

Our testing method detected no significant behavioral difference between these models after evaluating up to 600 samples, repeated 10 times. This suggests that few-shot prompting may have a more pronounced effect on larger models like Llama3-70B-Instruct compared to smaller ones like Llama3-8B-Instruct (Section~\ref{subsubsec:internal}). Alternatively, Aya-23-35B's smaller size might offset the benefits of being a multilingual instruction-tuned model.

\begin{wrapfigure}{R}{0.5\textwidth}
% \vspace{-1ex}
\resizebox{0.5\textwidth}{!}{
\begin{minipage}[H]{0.49\textwidth}
\centering
    \includegraphics[width=\textwidth]{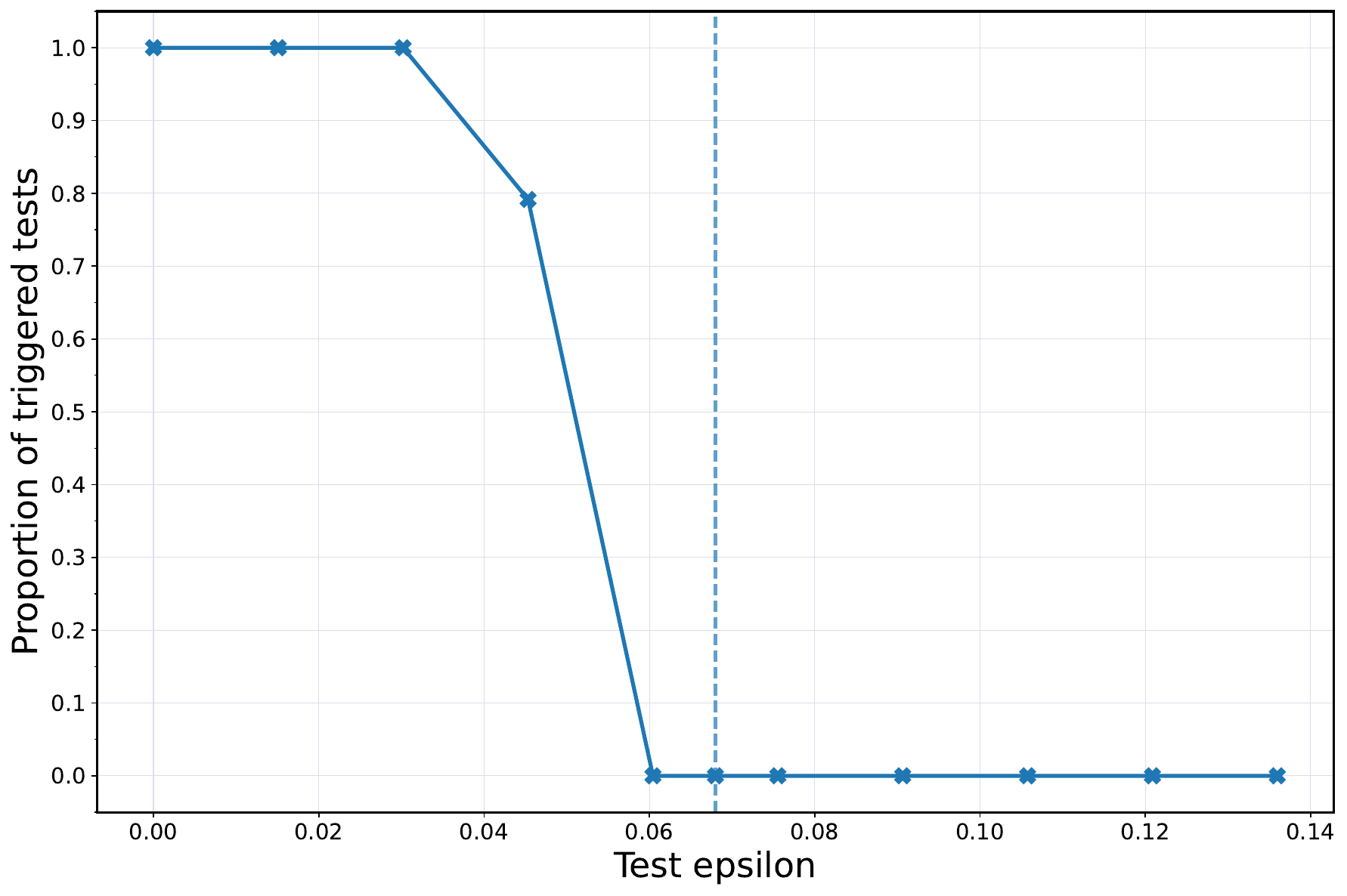}
    \caption{\textbf{Detection Rate over Test Epsilon.} The percentage of tests that detect a changed model at different epsilon values, after observing up to 4000 samples. Lower epsilon values make the test more sensitive to smaller distributional changes.}
    \label{fig:calibrated_detection_rate}
\end{minipage}
}
% \vspace{-1ex}
\end{wrapfigure}

\subsection{Comparison to Baselines}

To the best of our knowledge, our paper presents the first application of sequential hypothesis testing to the problem of detecting shift in model behavior, raising the question of an appropriate baseline to compare the performance of our proposed test. We give a brief overview of possible baselines and discuss some theoretical and practical reasons why our test is successful against them.

\begin{figure}[b]
    \centering
    \includegraphics[width=0.8\linewidth]{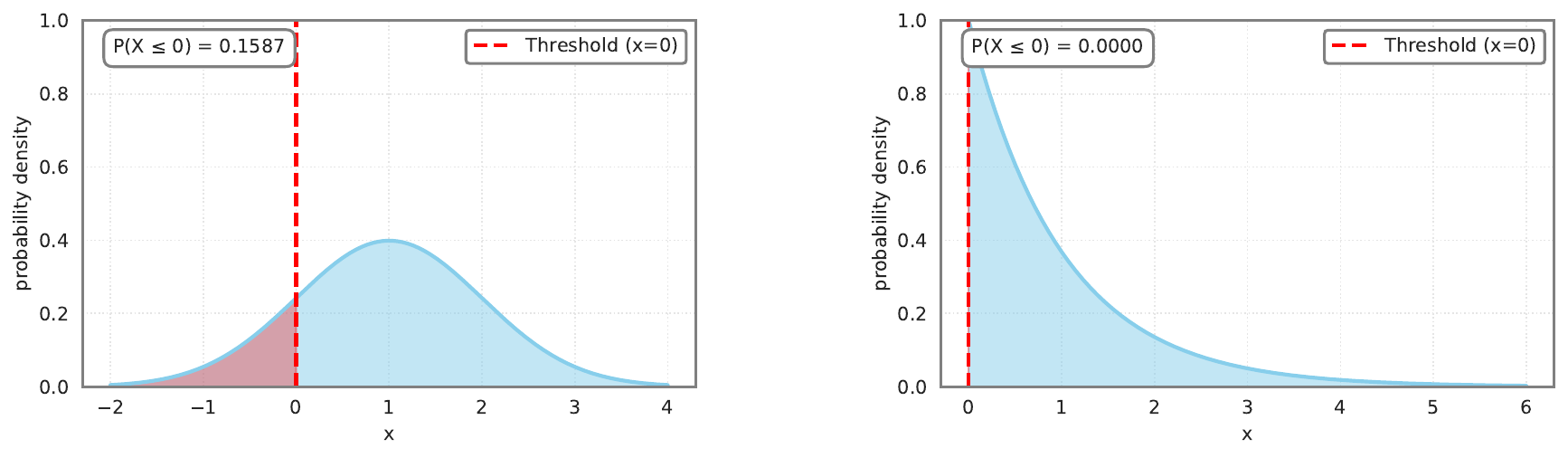}
    \caption{\textbf{Probability distributions with identical expected value and standard deviation can still differ in important ways}. Consider the example of a behavior, where we consider scores $<0$ as unsafe. Both the \textit{(Left)} normal distribution $\mathcal{N}(0,1)$ and the \textit{(Right)} Poisson distribution $\mathcal{P}_\lambda$ have $\mu=1$ and $\sigma^2=1$, but roughly $18\%$ of the probability mass of the normal distribution are below that threshold, vs. $0\%$ for the Poisson distribution.}
    \label{fig:same_mean_var}
\end{figure}

\paragraph{Summary Statistics.} Summary statistics such as mean and standard deviation are efficient in calculating and providing condensed information about a distribution. However, they might not capture some important aspects of behavior distributions. Consider \eg, the example in figure~\ref{fig:same_mean_var}, depicting two distributions with identical mean and standard deviation but whose \emph{tails} -- which might be particulary important for safety-critical behaviors -- look very different.

\paragraph{Distance Measures.} While distance measures such as Wasserstein distance take full distributions into account, we can only estimate them from samples. Given such an estimate, we lack a \emph{decision rule} to draw robust conclusions from the data about the true distance.

\paragraph{Classical Hypothesis Testing.}Unlike our method, classical hypothesis tests are not ``anytime-valid" -- meaning that we have to decide on a sample size before conducting a test or otherwise risk inflating the alpha error when including additional data~\citep{anscombe1954fixed}. We want to specifically consider the example of the two-sample Kolmogorov-Smirnov test that checks whether two samples come from the same distribution~\citep{pratt1981kolmogorov}.
Exacerbating the issue, the test is non-parametric, meaning that we cannot determine a sample size upfront via power analysis (\ie, based on the desired power and particular effect size) without making assumptions about the underlying distributions. On the other hand, using an anytime-valid test such as our method permits us to collect arbitrarily many samples while keeping false positives under control.

We conducted an experiment to study how repeated tests can lead to an inflated $\alpha$ error when using the Kolmogorov-Smirnov test versus our proposed method. We do this in the following way (presented in Algorithm~\ref{alg:ks_test}): During DAVT, whenever a new batch of data is collected, we not only update the wealth but also carry out a two-sample Kolmogorov-Smirnov test using all the available test data up until that point. Results for the three baseline models are depicted in table~\ref{tab:statistical-tests}. We find that repeated application of the Kolmogorov-Smirnov test leads to an inflated $\alpha$ for 2 out of the 3 models considered.

\begin{wrapfigure}{R}{0.5\textwidth}
\vspace{-1ex}
\resizebox{0.5\textwidth}{!}{
\begin{minipage}[H]{0.6\textwidth}
\begin{algorithm}[H] \caption{Repeated Kolmogorov-Smirnov Test} \label{alg:ks_test} \begin{algorithmic}[1] \STATE \textbf{Input:} $\{\rvx_i\}{i \geq 1}$ (stream of prompts), $B$ (behavior function), $\baselineM$ (baseline model API), $\changedM$ (current model API), $\alpha$ (type-I error limit under null), $n$ (batch size) \STATE Initialize empty lists: $\mathcal{B} \leftarrow \emptyset$, $\mathcal{B'} \leftarrow \emptyset$ \WHILE{true} \STATE Collect a batch of $n$ prompts: $\{\rvx_{t,i}\}_{i=1}^n$ \STATE Compute behavior scores for the batch: \FOR{$i = 1$ to $n$} \STATE $b_{t,i} \leftarrow B(\rvx_{t,i}, \baselineM(\rvx_{t_i}))$ \STATE $b'_{t,i} \leftarrow B(\rvx_{t,i}, \changedM(\rvx_{t,i}))$ \ENDFOR \STATE Append the batch scores to the lists: \STATE \quad $\mathcal{B} \leftarrow \mathcal{B} \cup \{b_{t,i}\}_{i=1}^n$ \STATE \quad $\mathcal{B'} \leftarrow \mathcal{B'} \cup \{b'_{t,i}\}_{i=1}^n$ \STATE Perform Kolmogorov-Smirnov Test on $\mathcal{B}$ and $\mathcal{B'}$: \STATE \quad Compute p-value $p_t \leftarrow \text{KS}(\mathcal{B}, \mathcal{B}')$ \IF{$p_t \leq \alpha$} \STATE Break and reject null hypothesis \ENDIF \ENDWHILE \end{algorithmic} \end{algorithm}
%      \end{minipage}
% }
% \vspace{-1ex}
% \end{wrapfigure}
     \end{minipage}
}
\vspace{-1ex}
\end{wrapfigure}

\begin{table}[t]
\caption{\textbf{Comparison of False Positive Rates for our proposed anytime-valid method and Kolmogorov-Smirnov Test}. Results show an increase in $\alpha$-error in 2 out of 3 cases when using the Kolmogorov-Smirnov test repeatedly on a growing number of batches while ours keeps it below $\alpha=5\%$. Runs were repeated 24 times, with each test running on up to 4000 samples and a batch size of 25.}
\label{tab:statistical-tests}
\begin{center}
\begin{tabular}{lcccc}
\multicolumn{1}{c}{\bf Test} & \multicolumn{1}{c}{\bf Llama3-8B-Instruct} & \multicolumn{1}{c}{\bf Mistral-7B-Instruct} & \multicolumn{1}{c}{\bf Gemma-1.1-7b}
\\ \hline \\
\textbf{Our Proposed Test} & 4.2\% & 0\% & 0\% \\
Kolmogorov Smirnov Test & 8.3\% & 0\% & 8.3\% \\
\end{tabular}
\end{center}
\end{table}

\subsection{Effects of Randomness and Errors in the Behavior Scoring Function}
\label{app:subsec:behavior_score}

\paragraph{Effects of Randomness.} The formulation of behavior shift auditing allows for the behavior scoring function to be a stochastic operator, as it is agnostic of the sources of variance in the distributions it compares, see Appendix~\ref{sec:derivation_and_proofs}. In the limit of infinite samples, the test result itself is unaffected by this randomness as long as the outputs of the stochastic behavior scoring function $\tilde{B}$ still reflect true scores in expectation \ie,
\begin{align*}
    B(\mathbf{x,y}) = \mathbb{E}[\tilde{B}(\mathbf{x,y})] \quad \text{ for every } (\mathbf{x,y})
\end{align*}
where $(\mathbf{x,y}) \in \mathcal{X} \times \mathcal{Y}$ denotes a (prompt, continuation)-pair. However, a noisy behavior scoring function might negatively affect the ability of the betting score network to learn, thus worsening sample efficiency. 

To investigate this, we repeat experiments from section \ref{subsec:exact_text}, modeling the stochasticity of $B$ by adding random Gaussian noise of different magnitudes to the scores from Perspective API.\footnote{Final toxicity scores are then clipped to the interval [0,1].} Figure~\ref{fig:power_over_noise} shows the fine-tuning detection rates for Llama3-8B-Instruct when using $\mathcal{N}(0,0.01)$, $\mathcal{N}(0, 0.05)$ and $\mathcal{N}(0, 0.1)$ noise.

\begin{figure}
    \centering
    \includegraphics[width=0.8\linewidth]{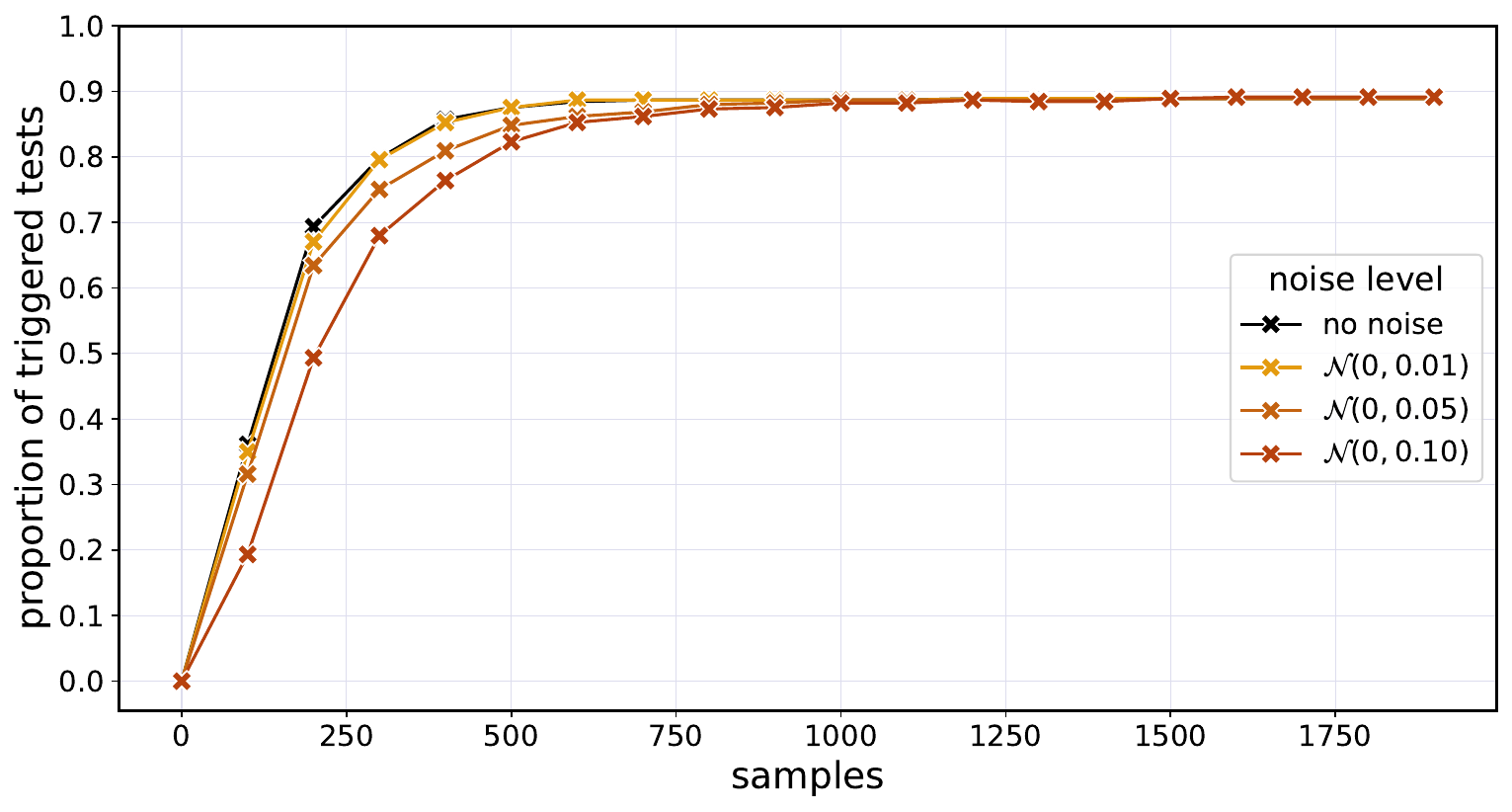}
    \caption{\textbf{Fine-tuning Detection for Llama3-8B-Instruct using noisy Scoring Functions.} The detection frequency as a function of number of generated samples. Each curve represents the average detection frequency over the 10 fine-tuning checkpoints produced in section \ref{subsec:exact_text}, but when using a scoring function with additional Gaussian noise.}
    \label{fig:power_over_noise}
\end{figure}

We find that sample efficiency decreases the more noise is added to toxicity scores. However, detection rates still eventually stabilize at the same rates as when using toxicity scores without additional noise.

\paragraph{Effects of Systematic Errors.} Our test is further robust against any bijective transformation in the behavior scoring function that could be recovered by the betting score network $\phi$, including scaling or consistent uniform under(over-)estimation.

\paragraph{Weak Proxies.} We call a scoring function $B_{\text{proxy}}$ ``weak proxy" for behavior $\mathcal{B}$ if it is correlated with the ground-truth scoring function $B$ on the available test data. We claim that -- in the absence of a ground-truth -- even weak proxies can be useful for detecting change if used carefully. The underlying rationale is that discrepancies in the distributions of ground-truth scores are likely to induce corresponding discrepancies in the distributions of proxy scores, provided there is a correlation between them. However, caution is warranted because positive test results may arise from changes in behaviors that are uncorrelated with the ground-truth scoring function. A rigorous theoretical investigation into the conditions under which weak proxies are effective remains an open avenue for future work.

\subsection{Extension to Multiple Behaviors}
\label{app:subsec:multiple_behaviors}

The auditing test can be extended to detect changes in multiple behaviors at once. The requirement for this is the existence of a dataset where all of the behaviors in question can be observed \ie, manifest with some non-zero probability.

The exact test is an application of DAVT, which \citet{DBLP:conf/aistats/PandevaFRS24} have successfully applied to multi-dimensional distributions. Assume we want to test for changes in $d$ behaviors as measured by behavior scoring functions $B_1, \dots, B_d$, producing the $d$-dimensional score 
\begin{align*}
    \boldsymbol{B}(X, M(X)):=(B_1(X, M(X)), \dots, B_d(X, M(X))
\end{align*}
In this case, the only modification necessary is the betting score network, with $\phi$ now taking in scores from $[0,1]^d$.

The generalization of the tolerance test to multiple behaviors is similarly straightforward if we decide to set a 
\emph{global} tolerance threshold $\epsilon>0$ as the maximal allowed difference between multi-dimensional distributions. Note that the derivation of the two-sample test with tolerance in Appendix~\ref{sec:derivation_and_proofs} does not depend on $X,Y$ being real-valued; we can instead define $\boldsymbol{X}:=(X_1, \dots, X_d), \boldsymbol{Y}:=(Y_1, \dots, Y_d) : \mathcal{X} \rightarrow [0,1]^d$. 

We might instead want to set \emph{separate} tolerance thresholds for different behaviors. The current version of our test does not allow for this. As an ad-hoc solution, we propose carrying out multiple tests on the same data in parallel and correcting for an increase in Type I error (\eg, using Bonferroni correction).

\end{document}